\documentclass[lettersize,journal]{IEEEtran}
\usepackage{amsmath,amsfonts}
\usepackage{algorithmic}
\usepackage{algorithm}
\usepackage{array}
\usepackage{textcomp}
\usepackage{stfloats}
\usepackage{verbatim}
\usepackage{graphicx}
\usepackage{cite}
\usepackage{balance}
 
\usepackage{mathtools}
\usepackage{graphics, amsfonts, graphicx, amssymb, cite, mathrsfs}
\usepackage{color}
\usepackage[dvipsnames]{xcolor}
\usepackage{setspace}
\usepackage{multirow}
\usepackage{rotating}
\usepackage{comment}
\usepackage[font=small]{caption}
\usepackage[affil-it]{authblk}
\usepackage{algorithm}
\usepackage{algorithmic}
\usepackage{enumitem}
\usepackage{cite}
\usepackage{booktabs}
\usepackage{subcaption}
\usepackage{lipsum}

\PassOptionsToPackage{hyphens}{url}\usepackage{hyperref}

\input{mysymbol.sty}

\usepackage{amsthm}

\newtheorem{assumption}{\hspace{0pt}\bf Assumption}
\newtheorem{lemma}{\hspace{0pt}\bf Lemma}

\newtheorem{theorem}{\hspace{0pt}\bf Theorem}

\newtheorem{remark}{\hspace{0pt}\bf Remark}

\newtheorem{definition}{\hspace{0pt}\bf Definition}

\date{\today}

\def\E{\mathbb{E}}

\definecolor{forestgreen}{rgb}{0.13, 0.55, 0.13}
\definecolor{Gray}{gray}{0.9}

\usepackage{todonotes}
\begin{document}

% \title{Unsupervised Learning of State-Augmented Wireless Resource Management Policies}
% \title{Learning State-Augmented Algorithms for \\ Resource Management in Wireless Networks}
\title{State-Augmented Learnable Algorithms for Resource Management in Wireless Networks}

\author{}

\author{Navid~NaderiAlizadeh,
        Mark~Eisen,
        and~Alejandro~Ribeiro% <-this % stops a space
\thanks{N. NaderiAlizadeh and A. Ribeiro are with the Department of Electrical and Systems Engineering, University of Pennsylvania, Philadelphia,
PA 19104, USA (e-mails: \{nnaderi, aribeiro\}@seas.upenn.edu). M. Eisen is with Intel Labs, Intel Corporation, Hillsboro, OR 97124, USA (e-mail: mark.eisen@intel.com).

This work was supported in part by ARL DCIST CRA under Grant W911NF-17-2-0181, the AI Institute for Learning-Enabled Optimization at Scale (TILOS) (NSF CCF-2112665), and the NSF-Simons Research Collaboration on the Mathematical and Scientific Foundations of Deep Learning (MoDL) (NSF DMS-2031985).

This paper has been presented in part at the 2022 Asilomar Conference on Signals, Systems, and Computers~\cite{naderializadeh2022state}.}% <-this % stops a space
}


% % column for its text to clear the IEEEpubid mark.

\maketitle

\begin{abstract}
We consider resource management problems in multi-user wireless networks, which can be cast as optimizing a network-wide utility function, subject to constraints on the long-term average performance of users across the network. We propose a state-augmented algorithm for solving the aforementioned radio resource management (RRM) problems, where, alongside the instantaneous network state, the RRM policy takes as input the set of dual variables corresponding to the constraints, which evolve depending on how much the constraints are violated during execution. We theoretically show that the proposed state-augmented algorithm leads to feasible and near-optimal RRM decisions. Moreover, focusing on the problem of wireless power control using graph neural network (GNN) parameterizations, we demonstrate the superiority of the proposed RRM algorithm over baseline methods across a suite of numerical experiments.
\end{abstract}


\section{Introduction}

With the proliferation of 5G network implementations across the globe and research already underway for future 6G wireless networks, novel wireless services and capabilities are expected to emerge that require carefully-optimized management of wireless resources. 
% Given the rising demand for wireless services with the advent of 5G networks, optimal resource allocation in wireless networks become of critical importance. 
Aside from traditional approaches for addressing such radio resource management (RRM) problems~\cite{shi2011iteratively, wu2013flashlinq, naderializadeh2014itlinq, yi2015itlinq+,shen2017fplinq}, learning-based methods have recently gained significant traction and demonstrated superior performance over prior approaches~\cite{eisen2019learning, nasir2019multi, liang2019deep, eisen2020optimal, shen2020graph, naderializadeh2021resource}. Such methods are envisioned to play a key role in current and future wireless networks with the ubiquitous availability of computational resources both at the end-user devices and within the network infrastructure~\cite{niknam2020intelligent,bonati2021intelligence,letaief2021edge,MediaTek_6G_whitepaper_2022}.

As a general formulation of the RRM problem, similarly to prior work~\cite{ribeiro2012optimal,eisen2019learning, eisen2020optimal, naderializadeh2022learning}, we consider a network utility maximization problem, subject to multiple constraints, where both the utility and the constraints are defined based on the long-term average performance of users across the network. A common method for solving such problems is to move to the Lagrangian dual domain, where a single objective, i.e., the Lagrangian, can be maximized over the primal variables and minimized over the dual variables, with each dual variable corresponding to a constraint in the original RRM problem. It can be shown that under mild conditions, such problems have null duality gap, hence allowing the use of primal-dual methods to derive optimal solutions. Even with the introduction of parameterizations for the RRM policy, the duality gap remains small in case of near-universal parameterizations, such as fully-connected deep neural networks. Nevertheless, such primal-dual algorithms lack feasibility guarantees---even for the case of convex optimization---except for an \emph{infinite} number of primal-dual iterations~\cite{nedic2009subgradient}. More precisely, it is unclear whether they lead to RRM decisions that satisfy the constraints in the original RRM problem if the training procedure is terminated after a finite number of iterations.

In this paper, we propose an alternative approach to solve the aforementioned constrained RRM problem. In particular, we leverage the fact that dual variables in constrained optimization provide indication of how much the corresponding constraints have been violated or satisfied over time. Using this observation, we incorporate the notion of \emph{state augmentation} from~\cite{calvo2021state}, in which the standard wireless network state is augmented with dual variables at each time instance to use as dynamic inputs to the RRM policy. By incorporating dual variables as an augmented input for the RRM policy, we are able to train the policy to adapt its decisions to not only instantaneous channel states, but also to such indications of constraint satisfaction. We theoretically establish two important properties of the proposed \emph{state-augmented} algorithm, which distinguish it from standard primal-dual training methods: the proposed algorithm leads to a trajectory of RRM decisions that are i) \emph{feasible}, i.e., satisfy the constraints almost-surely, and ii) \emph{near-optimal}, in the sense that the expected resulting network utility is within a constant additive gap of the optimum.

% We then dynamically update the dual variables during execution based on how much the corresponding constraints have been violated or satisfied. Borrowing  As shown in~\cite{calvo2021state}, such state augmentation is necessary when learning the policy directly is not necessarily possible.

As a prominent use case of the aforementioned state-augmented algorithm for solving RRM problems, we consider the problem of power control in multi-user interference channels, where the goal is to maximize the network sum-rate, subject to per-user minimum-rate requirements. We model the network state as the set of channel gains at each time step, and use a graph neural network (GNN) parameterization for the RRM policy, that takes as input both the network state and the dual variables at each time step, and outputs the transmit power levels. Through numerical experiments, we show that our proposed state-augmented algorithm, while slightly sacrificing the average performance, significantly outperforms baseline methods in terms of the worst-case user rates, thanks to satisfying the per-user minimum-rate constraints. We also show the benefits of the GNN-based parameterization of the RRM policy in terms of scalability to larger configurations and transferability to unseen network sizes, confirming the findings of prior work using such permutation-equivariant parameterizations~\cite{eisen2020optimal, lee2020graph, shen2020graph, naderializadeh2021wireless, chowdhury2021unfolding, wang2021unsupervised, chowdhury2021efficient, nikoloska2022modular, naderializadeh2022learning, li2022power, wang2022learning, zhao2022link}. It is important to note that GNNs are able to consider the relational structure among elements in an input graph when creating element-wise representations as opposed to other permutation-invariant/equivariant parameterizations in the set representation learning literature~\cite{zaheer2017deep, lee2019set, naderializadeh2021pooling, li2022heterogeneous}.

The rest of the paper is organized as follows. We begin in Section~\ref{sec:formulation} by formulating the radio resource management problem, in which instantaneous resource allocation decisions are made in response to network states to both maximize a utility and satisfy constraints on the long-term network performance. In Section~\ref{sec:param}, we show how parameterized RRM policies, trained via gradient-based methods, can lead to feasible RRM decisions that are optimal to within a constant additive gap by solving the problem in the Lagrangian dual domain. Such an algorithm is practically limited, however, due to, among other things, the large computational expense required during inference. 

In Section~\ref{sec:alg}, we describe our proposed approach of using state augmentation to learn alternative algorithms for RRM training and inference. We further prove its feasibility and near-optimality properties under proper assumptions on the expressive power of the state-augmented policy. In Section~\ref{sec:power_control}, we show the application of the proposed method to power control problems in multi-user interference channels with graph neural network parameterizations through an extensive series of numerical simulations. Finally, we conclude the paper in Section~\ref{sec:conclusion}.

\section{Problem Formulation}\label{sec:formulation}
We consider a wireless network operating over a series of time steps $t\in\{0,1,2,\dots,T-1\}$, where at each time step $t$, the set of channel gains in the network, or the \emph{network state}, is denoted by $\bbH_t \in \ccalH$. Given the network state, we let $\bbp(\bbH_t)$ denote the vector of \emph{radio resource management (RRM)} decisions across the network, where $\bbp : \ccalH \to \reals^a$ denotes the RRM function. These RRM decisions subsequently lead to the network-wide performance vector $\bbf(\bbH_t, \bbp(\bbH_t)) \in \reals^b$, with $\bbf: \ccalH \times \reals^a \to \reals^b$ denoting the performance function.

Given a concave utility $\mathcal{U}: \reals^b \to \reals$ and a set of $c$ concave constraints $\bbg: \reals^b \to \reals^c$, we define the generic RRM problem as
% \begin{subequations}\label{eq_param_problem}
% \begin{alignat}{2}
%     &\max_{\bbtheta,\bbx} &~~& \mathcal{U}(\bbx),             \\
%     &~~~\text{s.t.} && \bbx       \leq  \E_{\bbH} \left[ \bbf(\bbH, \bbp(\bbH;\bbtheta)) \right],   \label{eq:ergodic_rate_constraint}  \\
%     &&& \bbg(\bbx) \geq 0. \label{eq:min_rate_constraint_orig}%
% \end{alignat}
% \end{subequations}
\begin{subequations}\label{eq:nonparam_problem}
\begin{alignat}{2}
    &\max_{\{\bbp(\bbH_t)\}_{t=0}^{T-1}} &~~& \mathcal{U}\left( \frac{1}{T} \sum_{t=0}^{T-1} \bbf(\bbH_t, \bbp(\bbH_t)) \right),\label{eq:objective_non_param}             \\
    &~~~~~~\text{s.t.} &&  \bbg\left( \frac{1}{T} \sum_{t=0}^{T-1} \bbf(\bbH_t, \bbp(\bbH_t)) \right) \geq \bbzero,%
\end{alignat}
\end{subequations}
where the objective and the constraints are derived based on the \emph{ergodic average} network performance $\frac{1}{T} \sum_{t=0}^{T-1} \bbf(\bbH_t, \bbp(\bbH_t))$ rather than the instantaneous performance. The goal of the RRM problem is, therefore, to derive the optimal vector of RRM decisions $\bbp(\bbH_t)$ for any given network state $\bbH_t \in \ccalH$.

\section{Parameterized Gradient-Based RRM Algorithms in the Dual Domain}\label{sec:param}
As~\eqref{eq:objective_non_param} shows, problem~\eqref{eq:nonparam_problem} entails an infinite-dimensional search over the set of RRM decisions $\bbp(\bbH)$ for any given network state $\bbH$, which is practically infeasible. Therefore, we resort to \emph{parameterizing} the RRM policy and replacing $\bbp(\bbH)$ with $\bbp(\bbH;\bbtheta)$, where $\bbtheta \in \bbTheta$ denotes a finite-dimensional set of parameters. This, in turn, leads to the \emph{parameterized} RRM problem
\begin{subequations}\label{eq:param_problem}
\begin{alignat}{2}
    P^{\star}&=\max_{\bbtheta \in \bbTheta} &~~& \mathcal{U}\left( \frac{1}{T} \sum_{t=0}^{T-1} \bbf(\bbH_t, \bbp(\bbH_t;\bbtheta)) \right),             \\
    &~~~~~~\text{s.t.} &&   \bbg\left( \frac{1}{T} \sum_{t=0}^{T-1} \bbf(\bbH_t, \bbp(\bbH_t;\bbtheta)) \right)  \geq \bbzero, \label{eq:min_rate_constraint_param}%
\end{alignat}
\end{subequations}
where the maximization is now performed over the set of parameters $\bbtheta \in \bbTheta$.

In order to solve problem~\eqref{eq:param_problem}, we move to the Lagrangian dual domain. To derive the Lagrangian dual problem, we first introduce the Lagrangian function, with non-negative dual multipliers $\bbmu \in \reals_+^c$ associated with the constraints in \eqref{eq:min_rate_constraint_param}, as
\begin{align}
   \ccalL(\bbtheta, \bbmu) &= \mathcal{U}\left( \frac{1}{T} \sum_{t=0}^{T-1} \bbf(\bbH_t, \bbp(\bbH_t;\bbtheta)) \right)\nonumber \\
   &\qquad\qquad+ \bbmu^T \bbg\left( \frac{1}{T} \sum_{t=0}^{T-1} \bbf(\bbH_t, \bbp(\bbH_t;\bbtheta)) \right).\label{eq:param_lagrangian}
\end{align}
The Lagrangian in~\eqref{eq:param_lagrangian} can be optimized using gradient-based methods. In particular, we seek to maximize over $\bbtheta$, while subsequently minimizing over the dual multipliers $\bbmu$, i.e.,
\begin{align}\label{eq_dual_problem}
D^{\star} \coloneqq \min_{\bbmu \in \reals_+^c} \max_{\bbtheta \in \bbTheta} \ccalL(\bbtheta,\bbmu).
\end{align}
To train the model parameters $\bbtheta$, we introduce an \emph{iteration} duration $T_0$, which we define as the number of time steps between consecutive model parameter updates. Based on this notion, we define an iteration index $k\in\{0,1,2,\dots,K-1\}$ with $K=\lfloor T / T_0 \rfloor$, where the model parameters are updated as
\begin{align}
\bbtheta_k &= \arg \max_{\bbtheta\in\bbTheta} \Bigg[\mathcal{U}\left( \frac{1}{T_0} \sum_{t=kT_0}^{(k+1)T_0-1} \bbf(\bbH_{t}, \bbp(\bbH_{t};\bbtheta)) \right)\nonumber \\
  &\qquad\qquad+ \bbmu_{k}^T \bbg\left( \frac{1}{T_0} \sum_{t=kT_0}^{(k+1)T_0-1} \bbf(\bbH_{t}, \bbp(\bbH_{t};\bbtheta)) \right)\Bigg],\label{eq:theta_dynamics}
\end{align}
while the dual variables are updated recursively 
%at each time step $t\in\{0,1,2,\dots\}$ 
as
% based on which  we run the following dual gradient descent algorithm for each time step $t\in\{1,\dots,T\}$,
\begin{align}
% \bbtheta_t &= \arg \max_{\bbtheta\in\bbTheta} \ccalL(\bbtheta,\bbmu_t),\label{eq:theta_dynamics} \\
% \bbtheta_t &= \arg \max_{\bbtheta\in\bbTheta} \lim_{T\to\infty}\Bigg[\mathcal{U}\left( \frac{1}{T} \sum_{t'=t}^{t+T-1} \bbf(\bbH_{t'}, \bbp(\bbH_{t'};\bbtheta)) \right)\nonumber \\
%   &\qquad\qquad\qquad+ \bbmu_t^T \bbg\left( \frac{1}{T} \sum_{t'=t}^{t+T-1} \bbf(\bbH_{t'}, \bbp(\bbH_{t'};\bbtheta)) \right)\Bigg],\label{eq:theta_dynamics} \\
% \bbtheta_k &= \arg \max_{\bbtheta\in\bbTheta} \Bigg[\mathcal{U}\left( \frac{1}{T} \sum_{t'=kT}^{(k+1)T-1} \bbf(\bbH_{t'}, \bbp(\bbH_{t'};\bbtheta)) \right)\nonumber \\
%   &\qquad\qquad\qquad+ \bbmu_{kT}^T \bbg\left( \frac{1}{T} \sum_{t'=kT}^{(k+1)T-1} \bbf(\bbH_{t'}, \bbp(\bbH_{t'};\bbtheta)) \right)\Bigg],\label{eq:theta_dynamics} \\
% \bbmu_{t+1} &= \left[\bbmu_t - \eta_{\bbmu} \bbg\left(  \bbf(\bbH_t, \bbp(\bbH_t; \bbtheta_{\lfloor t / T_0 \rfloor})) \right)\right]_+,\label{eq:mu_dynamics}
\bbmu_{k+1} &= \left[\bbmu_k - \eta_{\bbmu} \bbg\left( \frac{1}{T_0} \sum_{t=kT_0}^{(k+1)T_0-1} \bbf(\bbH_{t}, \bbp(\bbH_{t};\bbtheta_k)) \right)\right]_+,\label{eq:mu_dynamics}
\end{align}
where $[\cdot]_+$ denotes projection onto the non-negative orthant, i.e., $[\cdot]_+ \coloneqq \max(\cdot, 0)$, and $\eta_{\bbmu}$ denotes the learning rate corresponding to the dual variables $\bbmu$.

Given the aforementioned training procedure, we can establish the following result on the generated RRM decisions. We note that similar results have been proven in~\cite{ribeiro2010ergodic, calvo2021state}.

% \newpage

\begin{theorem}\label{thm:main}
\allowdisplaybreaks
Consider a RRM algorithm, in which the primal parameters $\bbtheta$ and the dual variables $\bbmu$ follow the dynamics in~\eqref{eq:theta_dynamics} and~\eqref{eq:mu_dynamics}, respectively. Assuming that
\begin{itemize}
\item there exists a positive constant $G>0$ such that for any $\bbx \in \reals^b$, we have $|g_i(\bbx)| \leq G, \forall i\in\{1,\dots,c\}$; and,
\item there exists a strictly-feasible set of model parameters  $\hat{\bbtheta}$ such that $\bbg\left( \frac{1}{T} \sum_{t=0}^{T-1} \bbf(\bbH_t, \bbp(\bbH_t;\hat{\bbtheta})) \right) \geq G'\bbone$ for some positive constant $G'>0$,
\end{itemize}
then the resulting RRM decisions satisfy the desired constraints, i.e.,
\begin{align}\label{eq:thm_feasibility}
\lim_{T\to\infty}\bbg\left( \frac{1}{T} \sum_{t=0}^{T-1} \bbf(\bbH_t, \bbp(\bbH_t;\bbtheta_{\lfloor t/T_0 \rfloor})) \right) \geq \bbzero, \qquad a.s.
\end{align}
and they are within a constant additive gap of the optimum, i.e.,
\begin{align}
\lim_{T\to\infty} \hspace{-3pt} \E\left[ \mathcal{U}\left( \frac{1}{T} \sum_{t=0}^{T-1} \bbf(\bbH_t, \bbp(\bbH_t ;\bbtheta_{\lfloor t / T_0 \rfloor})) \right)\right]\hspace{-3pt} \geq\hspace{-2pt} P^{\star} \hspace{-3pt} - \frac{c\eta_{\bbmu}G^2}{2}.\label{eq:thm_optimality}
\end{align}

% {\color{teal}
% \begin{align}
% &\lim_{K\to\infty} \E\left[ \frac{1}{K} \sum_{k=1}^K \mathcal{U}\left( \frac{1}{T_0} \sum_{t=kT_0}^{(k+1)T_0-1} \bbf(\bbH_t, \bbp(\bbH_t ;\bbtheta_{k})) \right)\right] \nonumber\\
% &\geq P^{\star} - \frac{c\eta_{\mu}G^2}{2}.
% \end{align}
% }

% {\color{teal}
% \begin{align}
% \lim_{T\to\infty} \mathcal{U}\left( \frac{1}{T} \sum_{t=1}^T \E\left[ \bbf(\bbH_t, \bbp(\bbH_t;\bbtheta_t)) \right] \right) \geq P^{\star} - \frac{c\eta_{\mu}G^2}{2}.
% \end{align}
% }

\end{theorem}

\begin{proof}
See Appendix~\ref{appx:proof}.
\end{proof}

Theorem~\ref{thm:main} shows that the resulting RRM algorithm is both feasible and near-optimal if it is run for a large-enough number of time steps. This is notable, as with regular primal-dual methods, the RRM policies are not guaranteed to lead to a feasible set of RRM decisions, while such a feasibility guarantee exists for the RRM algorithm in~\eqref{eq:theta_dynamics}-\eqref{eq:mu_dynamics}. Note that the main factor that distinguishes~\eqref{eq:theta_dynamics}-\eqref{eq:mu_dynamics} from regular primal-dual methods (see, e.g.,~\cite{eisen2019learning}) is the maximization step in~\eqref{eq:theta_dynamics}, which adapts the RRM policy parameters to the dual variables at each iteration. % \blue{Mark: maybe it is worth stating what distinguishes (5)-(6) from "regular primal-dual methods" that gives it superior convergence properties, i.e. the maximization step in (5)?} \nn{Added a sentence on that; please check.}
The feasibility and near-optimality result of Theorem~\ref{thm:main} can indeed be proven for any non-convex optimization problem (see Appendix~\ref{appx:convex_hull}).

Nevertheless, the iterative RRM algorithm in~\eqref{eq:theta_dynamics}-\eqref{eq:mu_dynamics} presents a set of challenges that make it unsuitable for use in practice. First, maximizing the Lagrangian in~\eqref{eq:theta_dynamics} requires non-causal knowledge of the network states in the future%(see the definition of the Lagrangian in~\eqref{eq:param_lagrangian})
---the optimization at the $k$\textsuperscript{th} iteration (i.e., time step $t=kT_0$) requires the knowledge of the system state from $t=kT_0$ to $t=(k+1)T_0-1$, i.e., $\{\bbH_t\}_{t=kT_0}^{(k+1)T_0-1}$---which is impossible to obtain during execution, even though it might be viable during the training phase. Second, as Theorem~\ref{thm:main} demonstrates, convergence to  feasible and near-optimal RRM performance is only obtained as operation time $T$ tends to infinity. Thus, training iterations cannot be stopped at a finite time step, as there may not exist an iteration index $k$ for which $\bbtheta_k$ (or alternatively, a time average of $\{\bbtheta_{k'}\}_{k'=0}^{k-1}$) is optimal or feasible. %\green{This rules out the possibility of training the policy offline}\blue{Mark: is this right?}\nn{Technically yes, if we are not using a state-augmented parameterization. But since our proposed algorithm in the next section is indeed trained offline, it might be better to remove this point here.}.
Finally, in the maximization problem in~\eqref{eq:theta_dynamics}, the optimal set of model parameters needs to be found at each time step for a different vector of dual variables $\bbmu_k$, which can be computationally expensive, especially during the execution phase. This underscores the need to construct an algorithm that does not require memorization of the model parameters $\bbtheta$ for any given set of dual variables $\bbmu$, which we discuss next.

% Using a regular primal-dual approach, we can learn the parameters $\bbtheta$ and the dual variables $\bbmu$ over an iteration index $k$ as follows:
% \begin{align}
% \bbtheta_{k+1} &= \bbtheta_k + \eta_{\bbtheta} \nabla_{\bbtheta} \Bigg[ \mathcal{U}\left( \frac{1}{T} \sum_{t=1}^T \bbf(\bbH_t, \bbp(\bbH_t;\bbtheta)) \right) \nonumber \\
% &\qquad\qquad\qquad + \bbmu_k^T \bbg\left( \frac{1}{T} \sum_{t=1}^T \bbf(\bbH_t, \bbp(\bbH_t;\bbtheta)) \right) \Bigg]\\
% % \bbx_{k+1} &= \bbx_k + \eta_{\bbx} \left[\nabla_{\bbx} \left(\mathcal{U}(\bbx) + \bbmu_k^T \bbg(\bbx)\right) - \bblambda_k \right]\\
% % \bblambda_{k+1} &= [\bblambda_k - \eta_{\bblambda} (\E_{\bbH} \left[ \bbf(\bbH, \bbp(\bbH;\bbtheta_k)) \right] - \bbx_k)]_+\\
% \bbmu_{k+1} &= \left[\bbmu_k - \eta_{\bbmu} \bbg\left( \frac{1}{T} \sum_{t=1}^T \bbf(\bbH_t, \bbp(\bbH_t;\bbtheta_k)) \right)\right]_+,\label{eq:orig_dual_dynamics}
% \end{align}
% where $[\cdot]_+ \coloneqq \max(\cdot, 0)$, and $\eta_{\bbtheta}$ and $\eta_{\bbmu}$, respectively, denote the learning rates corresponding to the primal and dual variables, $\bbtheta$ and $\bbmu$.

\begin{remark}
It is important to note that in addition to the constraints on the ergodic average network performance, the formulation and aforementioned primal-dual algorithm are also able to handle instantaneous constraints if they are convex. However, they cannot address non-convex instantaneous constrains on the network performance. Considering non-convex instantaneous constraints is an interesting research direction, which we leave for future work.
\end{remark}

\section{Proposed State-Augmented RRM Algorithm}\label{sec:alg}

% \nn{To be revised based on iteration indexing ...}

In light of the aforementioned challenges, we propose a 
%In this paper, we take a different approach than the aforementioned primal-dual approach. Inspired by~\cite{calvo2021state}, we use the dual dynamics in~\eqref{eq:orig_dual_dynamics} to learn a 
\emph{state-augmented} RRM algorithm, similarly to~\cite{calvo2021state}, where the network state $\bbH_t$ at each time step $t$ is augmented by the corresponding set of dual variables $\bbmu_{\lfloor t/T_0 \rfloor}$, which are simultaneously used as inputs to the RRM policy. In particular, we introduce an alternative parameterization for the RRM policy, in which we represent the RRM decisions $\bbp(\bbH)$ using the parameterization $\bbp(\bbH, \bbmu; \bbphi)$, where $\bbphi\in\bbPhi$ denotes the set of parameters of the state-augmented RRM policy.

% For any time step $t\in\{1,\dots,T\}$, 
% In particular, during training, we first sample $\bbmu$ randomly from the non-negative orthant $\reals_+^c$. Given $\bbmu$, 
For a set of dual variables $\bbmu \in \reals_+^c$, we define the \emph{augmented} Lagrangian as
\begin{align}
\ccalL_{\bbmu}(\bbphi) &= \mathcal{U}\left( \frac{1}{T} \sum_{t=0}^{T-1} \bbf(\bbH_t, \bbp(\bbH_t, \bbmu;\bbphi)) \right) \nonumber \\
&\quad+ \bbmu^T \bbg\left( \frac{1}{T} \sum_{t=0}^{T-1} \bbf(\bbH_t, \bbp(\bbH_t, \bbmu;\bbphi)) \right).\label{eq:augmented_Lagrangian}
\end{align}
Then, considering a probability distribution $p_{\bbmu}$ for the dual variables, we define the optimal state-augmented RRM policy as that which maximizes the expected augmented Lagrangian over the distribution of all dual parameters, i.e., 
\begin{align}
\bbphi^{\star} &= \arg \max_{\bbphi \in \bbPhi} \E_{\bbmu \sim p_{\bbmu}}\left[ \ccalL_{\bbmu}(\bbphi) \right]. \label{eq:phi_dynamics_augmented}
\end{align}

Utilizing the state-augmented policy parameterized by $\bbphi^{\star}$ in \eqref{eq:phi_dynamics_augmented}, we can obtain the Lagrangian-maximizing RRM decision $\bbp(\bbH, \bbmu; \bbphi)$ for each dual iterate $\bbmu = \bbmu_{k}$. We therefore substitute the two-step update sequence in \eqref{eq:theta_dynamics}-\eqref{eq:mu_dynamics} with the state-augmented version of the dual multiplier update, i.e., %\blue{Mark: I think it might be a bit confusing to say (12) replaces (5), because (5) defines an iterative update whereas (12) is a separate (offline) optimization. Rather, by solving (12), we can replace (5)-(6) with (13).}\nn{Agreed with your point and your rewording of this paragraph.}
\begin{align}
% \bbmu_{t+1} &= \left[\bbmu_t - \eta_{\bbmu} \bbg\left(  \bbf(\bbH_t, \bbp(\bbH_t, \bbmu_t; \bbphi^{\star})) \right)\right]_+.\label{eq:mu_dynamics_augmented}
\bbmu_{k+1} \hspace{-3pt} &= \hspace{-3pt} \left[\bbmu_k - \eta_{\bbmu} \bbg\left( \frac{1}{T_0} \hspace{-6pt} \sum_{t=kT_0}^{(k+1)T_0-1} \bbf(\bbH_{t}, \bbp(\bbH_{t},\bbmu_k;\bbphi^{\star})) \right)\right]_+\hspace{-7pt}.\label{eq:mu_dynamics_augmented}
\end{align}
Note in \eqref{eq:mu_dynamics_augmented} the usage of the state-augmented RRM policy $\bbp(\bbH_{t},\bbmu_k;\bbphi^{\star})$ under the current dual multiplier at iteration $k$. Thus, in solving for the optimal state-augmented policy in \eqref{eq:phi_dynamics_augmented}, we are effectively \emph{learning} a parameterized model that, in conjunction with \eqref{eq:mu_dynamics_augmented}, executes the dual descent method in \eqref{eq:theta_dynamics}-\eqref{eq:mu_dynamics}.

Indeed, the main reason that we consider the state-augmented parameterization $\bbp(\bbH, \bbmu; \bbphi)$ is to resolve the challenge of memorizing the model parameters for any given set of dual variables. This implies that we need parameterizations with enough expressive power so that the RRM decisions made via the set of parameters $\bbphi^{\star}$ in~\eqref{eq:phi_dynamics_augmented} can closely approximate the ones using the iterative model parameters $\{\bbtheta_k (\bbmu_k)\}_{k=0}^{K-1}$ in~\eqref{eq:theta_dynamics}, where we have made explicit the dependence of $\bbtheta_k$ on $\bbmu_k$, $\forall k\in\{0,1,2,\dots,K-1\}$. To that end, we focus on a specific class of parameterizations, referred to as \emph{near-universal} parameterizations, which we define next.

\begin{definition}\label{def:near_universality}
Consider arbitrary functions $\bbtheta: \reals_+^c \to \bbTheta$ and $\bbp: \ccalH \times \bbTheta \to \reals^a$. A parameterization $\bbp(\bbH, \bbmu; \bbphi)$ with $\bbphi \in \bbPhi$ is near-universal with degree $\epsilon > 0$ for functions $\bbp(\bbH; \bbtheta(\bbmu))$ if for any network state $\bbH \in \ccalH$ and any functions $\bbtheta$ and $\bbp$, there exists $\bbphi \in \bbPhi$ such that
\begin{align}
\E_{\bbmu \sim p_{\bbmu}}\left\| \bbp(\bbH, \bbmu; \bbphi) - \bbp(\bbH; \bbtheta(\bbmu)) \right\|_{\infty} \leq \epsilon.
\end{align}
\end{definition}

We also introduce an additional assumption on the Lipschitz continuity of the expected utility and performance functions, which we mention below.
\begin{assumption}\label{assumption:Lipschitz}
The utility $\mathcal{U}$ and performance function $\bbf$ are expectation-wise Lipschitz. More precisely%, given arbitrary sets of network states $\{\bbH_t\}_{t=0}^{T-1}$% and RRM decisions $\{\bbp(\bbH_t)\}_{t=0}^{T-1}$
% , there exists a constant $L_{\mathcal{U}}$ such that 
, for any pair of ergodic average rate vectors $\bbx_1$ and $\bbx_2$, %RRM decision functions $\bbp_1$ and $\bbp_2$, 
we have
% \begin{align}
% & \E \left| \mathcal{U}\left( \sum_{t=0}^{T-1} \frac{\bbf(\bbH_t, \bbp_1(\bbH_t))}{T} \right) - \mathcal{U}\left( \sum_{t=0}^{T-1} \frac{\bbf(\bbH_t, \bbp_2(\bbH_t))}{T} \right) \right| \nonumber \\
% & \leq L_{\mathcal{U}} \E \left\| \frac{1}{T}\sum_{t=0}^{T-1} \bbf(\bbH_t, \bbp_1(\bbH_t)) - \bbf(\bbH_t, \bbp_2(\bbH_t)) \right\|_{\infty}.
% \end{align}
\begin{align}\label{eq:Lipschitz_U_1}
& \E \left| \hspace{1pt} \mathcal{U}\left( \bbx_1 \right) - \mathcal{U}\left( \bbx_2 \right) \right| \leq \E \left\| \bbx_1 - \bbx_2\right\|_{\infty}.
\end{align}
Moreover, given an arbitrary network state $\bbH$, there exists a constant $M$ such that any pair of RRM decision functions $\bbp_1$ and $\bbp_2$, we have
\begin{align}
& \E \left\| \bbf(\bbH, \bbp_1(\bbH)) - \bbf(\bbH, \bbp_2(\bbH)) \right\|_{\infty} \nonumber \\
& \qquad\qquad\qquad\leq M \E \left\| \bbp_1(\bbH) - \bbp_2(\bbH) \right\|_{\infty}.
\end{align}
\end{assumption}

Note that in~\eqref{eq:Lipschitz_U_1}, we assume that the Lipschitz constant for the utility $\mathcal{U}$ is equal to $1$ without loss of generality, as the utility function in the original optimization problem can be scaled by any arbitrary factor without any impact on the solution. Having Definition~\ref{def:near_universality} and Assumption~\ref{assumption:Lipschitz}, we can now state the following theorem, which shows that the RRM decisions generated by the proposed state-augmented procedure in~\eqref{eq:phi_dynamics_augmented}-\eqref{eq:mu_dynamics_augmented} are close to the ones made by the original iterative RRM algorithm in~\eqref{eq:theta_dynamics}-\eqref{eq:mu_dynamics}.

\begin{theorem}\label{thm:near_universality_result}
Under the hypotheses of Theorem~\ref{thm:main}, Assumption~\ref{assumption:Lipschitz}, and assuming that the state-augmented parameterization $\bbp(\bbH, \bbmu; \bbphi)$ is near-universal with degree $\epsilon$ in the sense of Definition~\ref{def:near_universality}, the RRM decisions made by the state-augmented algorithm in~\eqref{eq:phi_dynamics_augmented}-\eqref{eq:mu_dynamics_augmented} are both feasible, i.e.,
\begin{align}\label{eq:thm_feasibility_state_augmented}
\lim_{T\to\infty}\bbg\left( \frac{1}{T} \sum_{t=0}^{T-1} \bbf\left(\bbH_t, \bbp\left(\bbH_t, \bbmu_{\lfloor t/T_0 \rfloor}; \bbphi^{\star}\right) \right) \right) \geq \bbzero, \quad a.s.
\end{align}
and near-optimal, i.e.,
\begin{align}
&\lim_{T\to\infty} \E\left[ \mathcal{U}\left( \frac{1}{T} \sum_{t=0}^{T-1} \bbf\left(\bbH_t, \bbp\left(\bbH_t, \bbmu_{\lfloor t/T_0 \rfloor}; \bbphi^{\star}\right) \right) \right)\right] \nonumber \\
&\geq P^{\star} - \frac{c\eta_{\mu}G^2}{2} - M \epsilon.\label{eq:thm_optimality_state_augmented}
\end{align}
\end{theorem}

\begin{proof}
See Appendix~\ref{appx:proof_state_augmented}.
\end{proof}

\begin{algorithm*}[t]
\caption{Training Phase for the State-Augmented RRM Algorithm}
    \label{alg:training}
    \begin{algorithmic}[1]
    \STATE {\bfseries Input:} Number of training iterations $N$, batch size $B$, number of time steps $T$, primal learning rate $\eta_{\bbphi}$.
    \STATE Initialize: $\bbphi_0$.
    \FOR{$n=0, \ldots, N-1$}
        \FOR{$b=0, \ldots, B-1$}
            \STATE
            Randomly sample $\bbmu_b \sim p_{\bbmu}$.
            \STATE
            Randomly generate a sequence of network states $\{H_{b,t}\}_{t=0}^{T-1}$.
            \FOR{$t=0, \ldots, T-1$}
                \STATE 
                Generate RRM decisions $\bbp\left(\bbH_{b,t}, \bbmu_b; \bbphi_n\right)$.
            \ENDFOR
            \STATE
            Calculate the augmented Lagrangian according to~\eqref{eq:augmented_Lagrangian}, i.e., $$\ccalL_{\bbmu_b}(\bbphi_n) = \mathcal{U}\left( \frac{1}{T} \sum_{t=0}^{T-1} \bbf\left(\bbH_{b,t}, \bbp\left(\bbH_{b,t}, \bbmu_b; \bbphi_n\right)\right) \right)  + \bbmu^T \bbg\left( \frac{1}{T} \sum_{t=0}^{T-1}\bbf\left(\bbH_{b,t}, \bbp\left(\bbH_{b,t}, \bbmu_b; \bbphi_n\right)\right) \right).$$
        \ENDFOR
        \STATE
        Update the model parameters according to~\eqref{eq:phi_sga}, i.e., $$\bbphi_{n+1} = \bbphi_n +  \frac{\eta_{\bbphi}}{B} \sum_{b=0}^{B-1} \nabla_{\bbphi} \ccalL_{\bbmu_b}(\bbphi_n).$$
    \ENDFOR
    \STATE
    $\bbphi^{\star} \gets \bbphi_N$.
    \STATE
    {\bfseries Return:} {Optimal model parameters $\bbphi^{\star}$.}
    \end{algorithmic}
\end{algorithm*}

\begin{algorithm*}[t]
\caption{Execution Phase for the State-Augmented RRM Algorithm}
    \label{alg:execution}
    \begin{algorithmic}[1]
    \STATE {\bfseries Input:} Optimal model parameters $\bbphi^{\star}$, sequence of network states $\{H_{t}\}_{t=0}^{T-1}$, iteration length $T_0$, dual learning rate $\eta_{\bbmu}$.
    \STATE Initialize: $\bbmu_0 \gets \bbzero, k \gets 0$.
    \FOR{$t=0, \ldots, T-1$}
        \STATE 
        Generate RRM decisions $\bbp_t \coloneqq \bbp\left(\bbH_{t}, \bbmu_k; \bbphi^{\star}\right)$.
        \IF{$t+1 \mod T_0 = 0$}
        \STATE
        Update the dual variables according to~\eqref{eq:mu_dynamics_augmented}, i.e., $$\bbmu_{k+1} = \left[\bbmu_k - \eta_{\bbmu} \bbg\left( \frac{1}{T_0} \sum_{t=kT_0}^{(k+1)T_0-1} \bbf(\bbH_{t}, \bbp(\bbH_{t},\bbmu_k;\bbphi^{\star})) \right)\right]_+.$$
        \STATE
        $k \gets k+1.$
        \ENDIF
    \ENDFOR
    \STATE
    {\bfseries Return:} {Sequence of RRM decisions $\{\bbp_t\}_{t=0}^{T-1}$.}
    \end{algorithmic}
\end{algorithm*}

% \nn{Near-universality assumption + theorem to be added.}

\subsection{Practical Considerations}% Implementation of the State-Augmented Training Procedure}

% While the maximization in~\eqref{eq:phi_dynamics_augmented} makes the dependence of the model parameters $\bbphi$ on the dual variables $\bbmu$ more structured, it still requires non-causal access to the future states, which is infeasible during execution. 
In order to resolve the maximization in~\eqref{eq:phi_dynamics_augmented} during the offline training phase, we use a gradient ascent-based approach to learn an optimal set of parameters $\bbphi^{\star}$, which can be frozen after training is complete and utilized during execution. In particular, during the training phase, for a batch of dual variables $\{\bbmu_b\}_{b=1}^B$, randomly sampled from the distribution $p_{\bbmu}$, we consider the \emph{empirical} version of the Lagrangian maximization problem in~\eqref{eq:phi_dynamics_augmented}, i.e.,
\begin{align}
\bbphi^{\star} &= \arg \max_{\bbphi \in \bbPhi} \frac{1}{B} \sum_{b=1}^B \ccalL_{\bbmu_b}(\bbphi), \label{eq:phi_training_in_practice}
\end{align}
which we iteratively solve using gradient ascent. More specifically, we randomly initialize the model parameters as $\bbphi_0$, and then update them over an iteration index $n\in\{0,1,2,\dots,N-1\}$ as
\begin{align}
\bbphi_{n+1} &= \bbphi_n +  \frac{\eta_{\bbphi}}{B} \sum_{b=0}^{B-1} \nabla_{\bbphi} \ccalL_{\bbmu_b}(\bbphi_n),\label{eq:phi_sga}
\end{align}
where $\eta_{\bbphi}$ denotes the learning rate corresponding to the model parameters $\bbphi$. With a slight abuse of notation, we denote the final set of model parameters (i.e., the outcome of~\eqref{eq:phi_sga} upon convergence) as $\bbphi^{\star}$. Note that to enhance the generalization capability of the trained model, for each set of dual variables $\bbmu_b, b\in\{0,1,2,\dots,B\}$, in the batch, we can also randomly sample a separate realization of the sequence of network states, i.e., $\{H_{b,t}\}_{t=0}^{T-1}$, from the underlying random process, allowing the model to be optimized over a \emph{family} of network realizations. The training procedure is summarized in Algorithm~\ref{alg:training}.

%\nn{We could also talk about a randomly sampled batch of channel gains at each training iteration (i.e., optimizing over a \emph{family} of network configurations)?}

% for each configuration, we  we  we follow the following procedure:
% \begin{enumerate}
%     \item Sample $\bbmu$ randomly from the non-negative orthant $\reals_+^c$. {\color{red}(better sampling method? maybe higher $\mu$ for users in worse channel conditions?%normalizing the GNN inputs?
%     )}
%     \item Maximize the augmented Lagrangian $$\ccalL_{\bbmu}(\bbtheta) = \mathcal{U}\left( \frac{1}{T} \sum_{t=1}^T \bbf(\bbH_t, \bbp(\bbH_t, \bbmu;\bbtheta)) \right) + \bbmu^T \bbg\left( \frac{1}{T} \sum_{t=1}^T \bbf(\bbH_t, \bbp(\bbH_t, \bbmu;\bbtheta)) \right)$$ using gradient ascent on $\bbtheta$.
% \end{enumerate}

% \subsection{Dual-Augmented Execution Procedure}

During execution, the dual dynamics are used to update the dual variables and generate the RRM decisions as follows: We initialize the dual variables as $\bbmu_0 = \bbzero$. For any time step $t\in\{0,1,2,\dots,T-1\}$, given the network state $\bbH_t$, we generate the RRM decisions using the state-augmented RRM policy $\bbp(\bbH_t, \bbmu_{\lfloor t / T_0 \rfloor};\bbphi^{\star})$. Then, we update the dual variables every $T_0$ time steps as in~\eqref{eq:mu_dynamics_augmented}.
% \begin{align}
% \bbmu_{t+1} &= \left[\bbmu_t - \eta_{\bbmu} \bbg\left(  \bbf(\bbH_t, \bbp(\bbH_t, \bbmu_t;\bbtheta)) \right)\right]_+.\label{eq:dual_dynamics}
% \end{align}
Note how the dual dynamics in~\eqref{eq:mu_dynamics_augmented} track the satisfaction of the original constraints in~\eqref{eq:min_rate_constraint_param}. In particular, if the RRM decisions at time step $t$ help satisfy the constraints, the dual variables are reduced. On the other hand, if the constraints are not satisfied at a given time step, the dual variables increase in value. The state-augmented execution procedure is summarized in Algorithm~\ref{alg:execution}.

% \begin{enumerate}
%     \item Initialize the dual variables $\bbmu_1 = \bbzero$.
%     \item for $t=1,2,\dots$, do:
%     \begin{itemize}
%         \item Generate channel $\bbH_t$.
%         \item Generate the RRM decisions $\bbp(\bbH_t, \bbmu_t;\bbtheta)$.
%         \item Update the dual variables:
%         \begin{align}
%         % \bbx_{t+1} &= \bbx_t + \eta_{\bbx} \left[\nabla_{\bbx} \left(\mathcal{U}(\bbx) + \bbmu_k^T \bbg(\bbx)\right) - \bblambda_t \right]\\
%         % \bblambda_{t+1} &= \left[\bblambda_t - \eta_{\bblambda} \left(\frac{1}{t} \sum_{\tau=1}^{t} \left[ \bbf(\bbH_{\tau}, \bbp(\bbH_{\tau}, \bblambda_{\tau}, \bbmu_{\tau};\bbtheta)) \right] - \bbx_t\right)\right]_+\\
%         % \bbmu_{t+1} &= [\bbmu_t - \eta_{\bbmu} g(\bbx_t)]_+.\\
%         \bbmu_{t+1} &= \left[\bbmu_t - \eta_{\bbmu} \bbg\left(  \bbf(\bbH_t, \bbp(\bbH_t, \bbmu_t;\bbtheta)) \right)\right]_+.\label{eq:dual_dynamic}
%         \end{align}
%     \end{itemize}
% \end{enumerate}

\begin{remark}
Note that during the training procedure, the choice of distribution $p_{\bbmu}$ for sampling the dual variables, alongside the fact that they are kept fixed for the entire $T$ time steps, can affect the optimal model parameters $\bbphi^{\star}$. As we will show in Section~\ref{sec:exp_results}, a uniform distribution provides a desirable performance in our experiments. However, according to the dual dynamics in~\eqref{eq:mu_dynamics_augmented}, sampling the dual variables from the dual descent trajectory might lead to superior performance. Such a distribution may also be dependent on the realization of the initial network state, i.e., $\bbH_0$. For example, in an interference channel, the dual variables are expected to be higher on average for users in poor channel conditions (i.e., those with low signal-to-noise ratio (SNR) or large incoming interference-to-noise ratio (INR) values) as compared to users in favorable channel conditions (i.e., those with high SNR or low incoming INR values). The problems of finding the best distribution for sampling dual variables during training, and whether and how to update them during the $T$ training time steps for a given network realization, are interesting research directions, which we leave as future work.
\end{remark}

\begin{remark}
The RRM decisions $\bbp\left(\bbH_t, \bbmu_{\lfloor t/T_0 \rfloor}; \bbphi^{\star}\right)$ generated by our proposed algorithm might be suboptimal before the convergence of the dual variables. This implies that it might be beneficial to either i) have a number of warm-up time steps at the beginning of the execution phase until optimal RRM decisions are generated, or ii) initialize the dual variables with their optimal levels, which can be possible for network realizations seen during training, but does not necessarily transfer across different network realizations.
\end{remark}

\pagebreak

\section{Application to Power Control with \\ Graph Neural Network Parameterizations}\label{sec:power_control}
As an important application of the RRM problem to showcase the capabilities of the proposed state-augmented learning algorithm, we consider the problem of power control in $m$-user interference channels. More precisely, we study networks comprising $m$ \emph{users}, i.e., transmitter-receiver pairs $\{(\mathsf{Tx}_i, \mathsf{Rx}_i)\}_{i=1}^m$, where each transmitter intends to communicate to its corresponding receiver, while causing interference at other receivers. In this setting, the network state at time step $t$, i.e., $\bbH_t\in\mathbb{C}^{m \times m}$, contains all channel gains in the network, where the channel gain between transmitter $\mathsf{Tx}_i$ and receiver $\mathsf{Rx}_j$ at time step $t$ is denoted by by $h_{ij, t}\in\mathbb{C}$. The RRM decisions $\bbp \in [0, P_{\max}]^m$ represent the transmit power levels of the transmitters, with $P_{\max}$ denoting the maximum transmit power. Moreover, the performance function $\bbf(\bbH_t, \bbp)\in\reals^m$ represents the receiver rates, where for each receiver $\mathsf{Rx}_i$, the rate at time step $t$ is given by
\begin{align}
f_i(\bbH_t, \bbp) = \log_2\left(1+\frac{p_i \left|h_{ii,t}\right|^2}{N + \sum_{j=1, j\neq i}^m p_j \left|h_{ji,t}\right|^2}\right),
\end{align}
where $N$ denotes the noise variance, and it is assumed that the receiver treats all incoming interference as noise~\cite{annapureddy2009gaussian,geng2015optimality}.

We use a sum-rate utility $\ccalU(\bbx)=\sum_{i=1}^m x_i$, and consider per-user minimum-rate requirements as the constraints. In particular, we consider $c=m$ constraints, where the $i$\textsuperscript{th} constraint is given by $g_i(\bbx) = x_i - f_{\min}$, with $f_{\min}$ denoting the minimum per-user rate requirement. Therefore, letting $\bbone_m$ denote an $m$-dimensional vector of $1$'s, the optimization problem in~\eqref{eq:nonparam_problem} can be rewritten as
\begin{subequations}\label{eq:nonparam_problem_power_control}
\begin{alignat}{2}
    &\max_{\{\bbp(\bbH_t)\}_{t=0}^{T-1}} &~~& \frac{1}{T} \sum_{t=0}^{T-1} \sum_{i=1}^m f_i(\bbH_t, \bbp(\bbH_t)) ,\label{eq:objective_non_param_power_onctrol}             \\
    &~~~~~~\text{s.t.} &&   \frac{1}{T} \sum_{t=0}^{T-1} \bbf(\bbH_t, \bbp(\bbH_t))  \geq f_{\min} \bbone_m,\label{eq:min_rate_constraint}\\
    &~~~~~~ &&  \bbp(\bbPi^T \bbH \bbPi) = \bbPi^T \bbp(\bbH), \forall \bbH\in\ccalH, \forall \bbPi\in\ccalS_m.\label{eq:perm_eq_constraint}%
\end{alignat}
\end{subequations}
Note that in~\eqref{eq:perm_eq_constraint}, we have explicitly added a constraint on the \emph{permutation equivariance} of the power control policy, with $\ccalS_m$ denoting the set of all $m \times m$ permutation matrices, i.e.,
\begin{align}
\ccalS_m \coloneqq \left\{ \bbPi \in \{0,1\}^{m \times m} : \bbPi \bbone_m = \bbPi^T \bbone_m = \bbone_m \right\}.
\end{align}
This motivates the use of graph neural network (GNN) parameterizations, since they satisfy the permutation equivariance architecture by definition. For the problem of processing signals on graphs, certain GNN architectures have been shown to satisfy the near-universality assumption for a class of continuous, equivariant functions~\cite{keriven2019universal, azizian2021expressive, keriven2021universality}. It follows that with the restriction in~\eqref{eq:perm_eq_constraint}, the hypotheses of Theorem~\ref{thm:near_universality_result} are satisfied, since the set of permutation-equivariant functions is convex.% These universality results are applicable to our setting, since the set of permutation-equivariant functions is convex.

\subsection{GNN-Based Parameterizations}

In order to parameterize the state-augmented RRM policy, as in prior work~\cite{shen2020graph, eisen2020optimal, lee2020graph, naderializadeh2022learning}, we use GNN architectures, which have been shown to provide several benefits, such as scalability and transferability, alongside their inherent permutation equivariance. We operate the GNN over a graph-structured format of the interference channel, defined as a graph $\ccalG_t=(\ccalV, \ccalE, \bby_t, w_t)$ at each time step $t$, where:
\begin{itemize}
\item $\ccalV=\{1,2,\dots,m\}$ denotes the set of graph nodes, with each node representing a user (i.e., transmitter-receiver pair) in the graph,
\item $\ccalE = \ccalV \times \ccalV$ denotes the set of directed graph edges,
\item $\bbY_t \in \reals^{m \times 1}$ denotes the initial node features, which we set to the dual variables associated to the users at each time step, i.e., $\bbY_t \coloneqq \bbmu_{\lfloor t / T_0 \rfloor}$, and
\item $w_t: \ccalE \to \reals$ denotes the function that maps each edge to its weight at time $t$, which we define as $w_t(i, j) \coloneqq \frac{1}{Z_t} \log\left(P_{\max}|h_{ij, t}|^2 / N\right)$, with $Z_t > 0$ representing a normalization factor.
\end{itemize}

Note how the aforementioned graph structure allows the implementation of the state-augmented architecture, where the GNN can take the dual variables corresponding to all the users as the input node features, and the network state as the input edge features. More precisely, we consider $L$ consecutive GNN message-passing layers, where the $l$\textsuperscript{th} layer, $l\in\{1,2,\dots,L\}$ transforms the node features at layer $l-1$, i.e., $\bbY_t^{l-1}\in\reals^{m \times F_{l-1}}$, to the node features at layer $l$, i.e., $\bbY_t^{l}\in\reals^{m \times F_{l}}$. In its most general form, for each node $v\in\ccalV$, the aforementioned update can be written as
\begin{align}
\bby_{v,t}^{l} = \bbPsi^l\left( \left\{\bby_{u,t}^{l-1}, w_t(u,v) \right\}_{u\in\ccalV: (u,v)\in\ccalE}; \bbphi^l\right),
\end{align}
where $\bbPsi^l(\cdot; \bbphi^l): \reals^{F_{l-1}} \to \reals^{F_{l}}$ denotes the message-passing GNN operator, parameterized by the set of parameters $\bbphi^l$. We use $\bbY_t^0 = \bbY_t$ (with $F_0=1$) as the initial node features. We also set $F_{L}=1$ to produce a scalar output feature per node, and we define the power control decisions as
\begin{align}\label{eq:final_power_levels_GNN}
\bbp\left(\bbH_t, \bbmu_{\lfloor t / T_0 \rfloor};\bbphi\right) \coloneqq P_{\max} \cdot \sigma\left( \bbY_t^L \right),
\end{align}
where $\bbphi \coloneqq \left\{\bbphi^l\right\}_{l=1}^L$, and $\sigma(\cdot)$ denotes element-wise sigmoid function, with $\sigma(x)=\frac{1}{1+\exp(-x)}, \forall x\in\reals$, which is used to respect the instantaneous power level constraints $\bbp \in [0, P_{\max}]^m$.

\begin{figure*}[h]
\centering
\includegraphics[width=\textwidth]{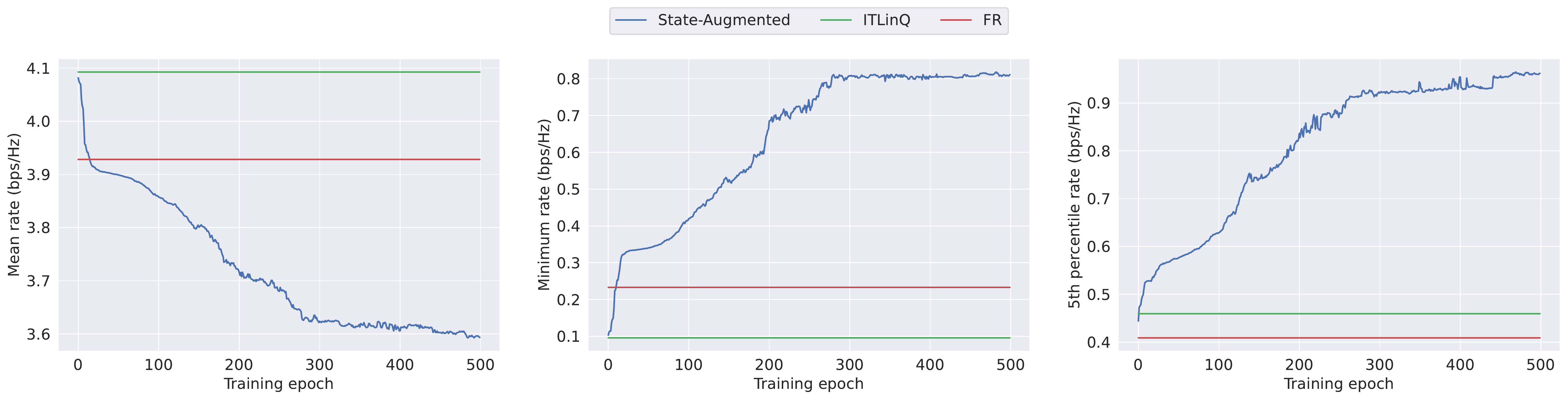}
\caption{Convergence behavior of the proposed state-augmented RRM algorithm, and its comparison with the baseline methods for a \emph{single} network realization with $m=50$ transmitter-receiver pairs (variable-density scenario). Note that the baseline algorithms fail to achieve the minimum-rate requirement of $f_{\min}=0.6$, while our proposed method is able to satisfy the constraints for all $50$ users.
}
\label{fig:convergence_singleConfig_m50_vardensity}
\end{figure*}

\begin{figure*}[h]
\centering
\includegraphics[width=\textwidth]{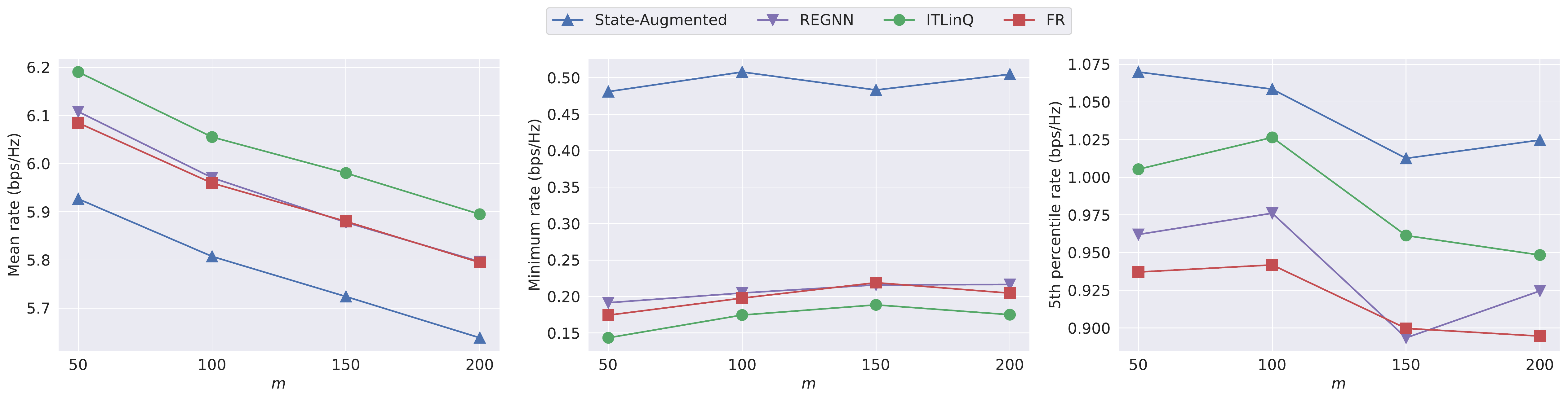}
\caption{Performance comparison of the proposed state-augmented RRM algorithm against baseline methods in the fixed-density scenario for networks with $m\in\{50,100,150,200\}$ transmitter-receiver pairs.}
\label{fig:scalability_fixed_density}
\end{figure*}

\subsection{Experimental Results}\label{sec:exp_results}
We generate networks with $m$ transmitter-receiver pairs, located randomly within a square network area of side length $R$. We drop the transmitters uniformly at random within the network area, while ensuring a minimum distance of $75$m between each pair of them. Afterwards, for each transmitter, we drop its associated receiver uniformly at random within an annulus around the transmitter, with inner and outer radii of $10$m and $50$m, respectively. To control the network density, we consider the following two scenarios to determine the network area size:
\begin{itemize}
    \item \textbf{Fixed Density:} We set $R = \sqrt{m / 20} \times 2$km to keep the density constant (at $5$ users/km$^2$).
    \item \textbf{Variable Density:} We set $R=2$km, implying that the network density increases with $m$.
\end{itemize}
In both cases, and for all values of $m$, we set $f_{\min} = 0.6$bps/Hz, $T_0=5$, and $T=100$. We consider both large-scale and small-scale fading for the channel model. The large-scale fading follows a dual-slope path-loss model similar to~\cite{zhang2015downlink,andrews2016we,naderializadeh2022learning} alongside a $7$dB log-normal shadowing. Moreover, the small-scale fading models channel variations across different time steps following a Rayleigh distribution with a pedestrian speed of $1$m/s~\cite{li2002simulation}. We set the maximum transmit power to $P_{\max}=10$dBm, and the noise variance to $N=-104$dBm (due to the $10$MHz bandwidth and noise power spectral density of $-174$dBm/Hz).

We use a $3$-layer GNN with $F_1=F_2=64$, where the first two layers are based on the local extremum operator proposed in~\cite{ranjan2020asap}, while the last layer entails a linear projection (together with the mapping in~\eqref{eq:final_power_levels_GNN}). We set the primal and dual learning rates to $\eta_{\bbphi} = 10^{-1} / m$ and $\eta_{\bbmu} = 20$, respectively, and we set the batch size to $B=128$. For each value of $m$, we generate a total of $256$ training samples and $128$ samples for evaluation, where each sample refers to a realization of the transmitter/receiver locations (and the large-scale fading), alongside the small-scale fading random process. We set the normalization factor for edge weights at time step $t$ to $Z_t = \left\|\log\left(P_{\max}|\bbH_t|^2 / N\right)\right\|_2$. Except for the case of Section~\ref{sec:single_network} below, we run training for $100$ epochs, and we draw the dual variables during training randomly from the $U(0,1)$ distribution.

\urlstyle{tt}
As baselines, we compare the performance of our proposed method against three baselines: i) Full reuse (FR), where every transmitter uses $P_{\max}$; ii) Information-theoretic link scheduling (ITLinQ)~\cite{naderializadeh2014itlinq}, and iii) random-edge graph neural networks (REGNN)~\cite{eisen2020optimal}.\footnote{For a fair comparison with~\cite{eisen2020optimal}, we use the same exact GNN structure and hyperparameters as for our method to implement the REGNN method.} We would like to point out that, while the REGNN baseline allows the consideration of minimum-rate constraints, the first two baselines (i.e., FR and ITLinQ) are not able to handle such per-user requirements. In what follows, we primarily report the results in terms of three separate metrics, namely the mean rate, minimum rate (discarding the bottom 1\% of user rates as outliers), and the 5\textsuperscript{th} percentile rate, evaluated over the test configurations.\footnote{Our code is available at \url{https://github.com/navid-naderi/StateAugmented_RRM_GNN}.}

\begin{figure*}[h]
\centering
\includegraphics[width=\textwidth]{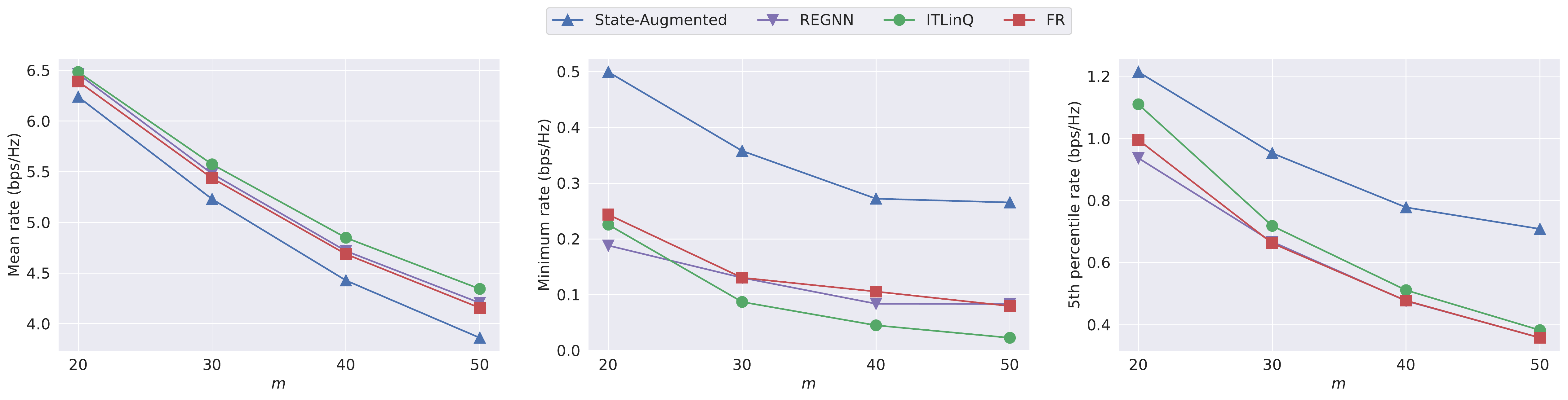}
\caption{Performance comparison of the proposed state-augmented RRM algorithm against baseline methods in the variable-density scenario for networks with $m\in\{20, 30, 40, 50\}$ transmitter-receiver pairs.}
\label{fig:scalability_variable_density}
\end{figure*}

\begin{figure*}[t!]
    \centering
    \begin{subfigure}[t]{0.5\textwidth}
        \centering
        \includegraphics[width=.95\textwidth]{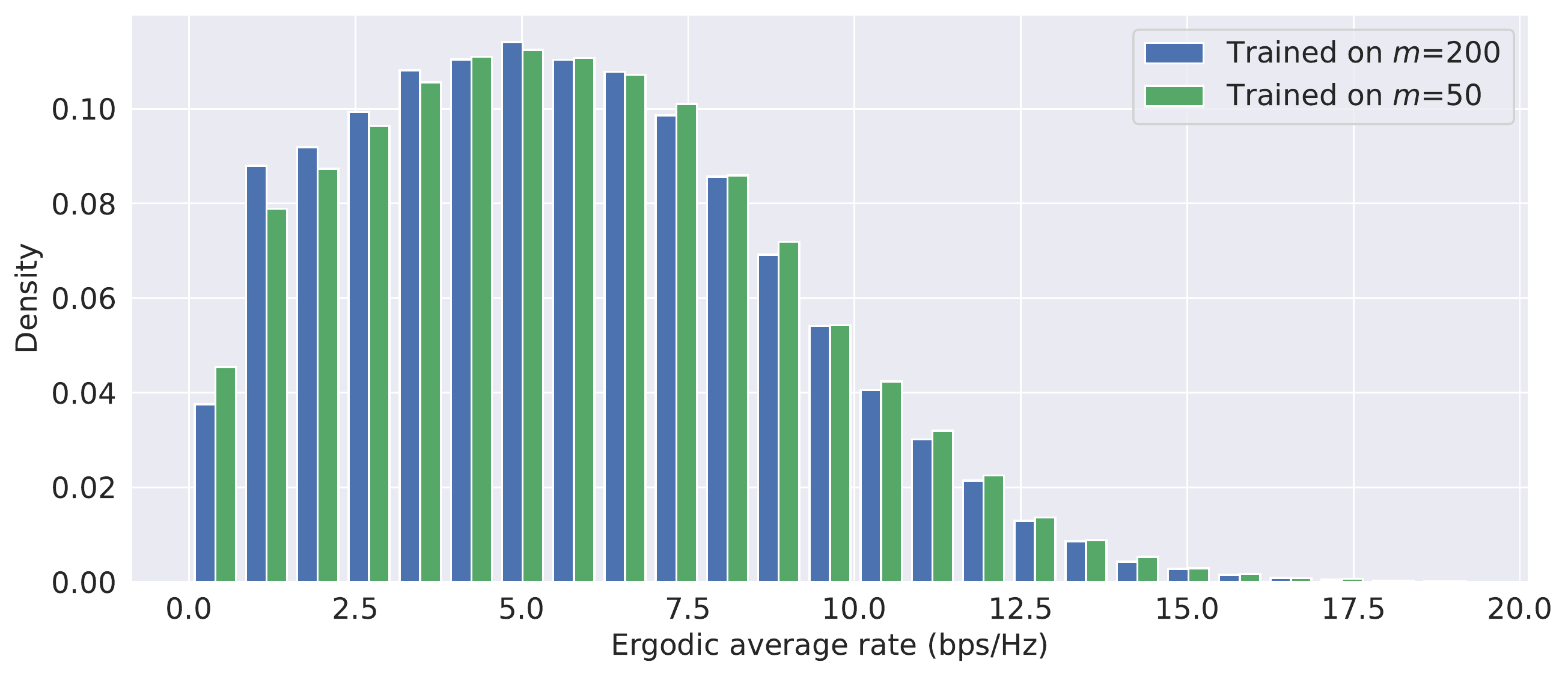}
        \caption{\hspace*{-.25in}}
        \label{fig:transfer_fixed_density_50To200}
    \end{subfigure}%
    \hfill
    \begin{subfigure}[t]{0.5\textwidth}
        \centering
        \includegraphics[width=.95\textwidth]{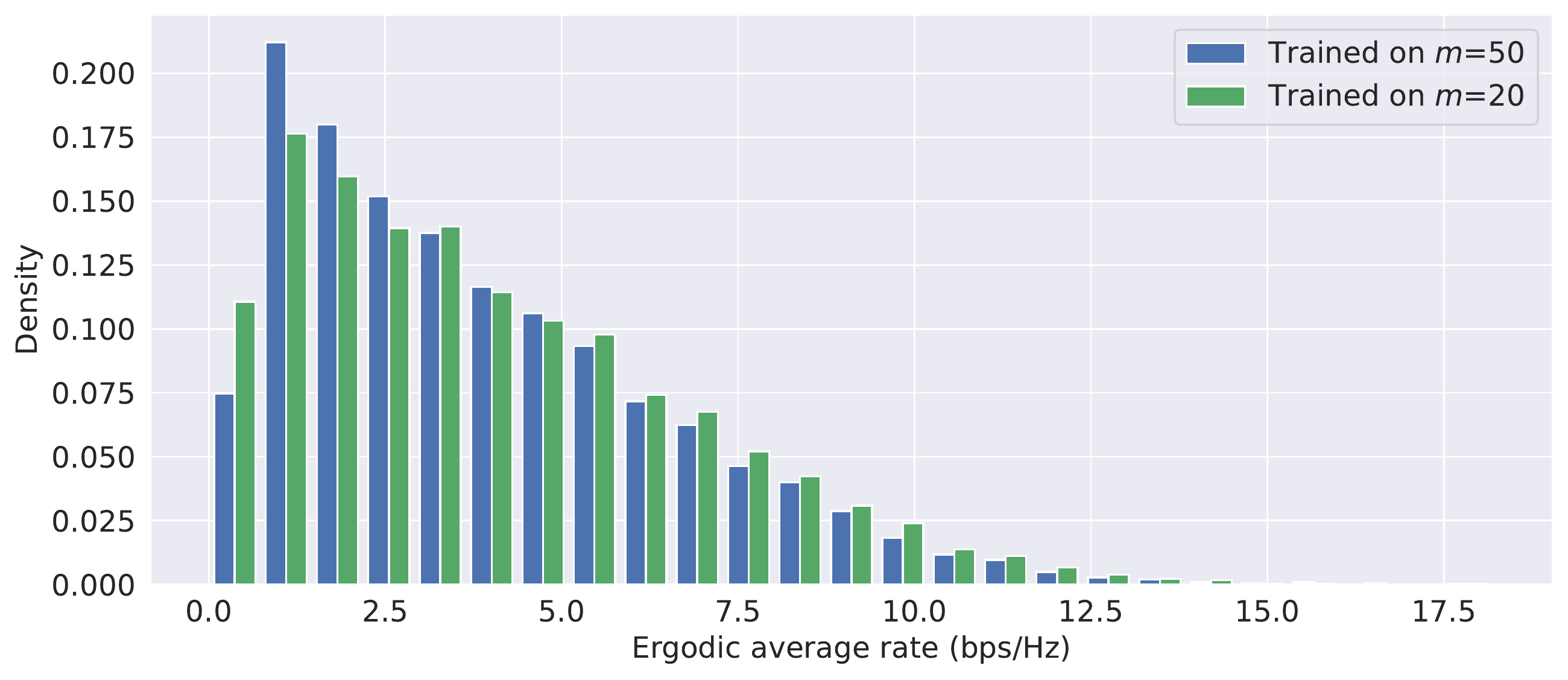}
        \caption{\hspace*{-.25in}}
        \label{fig:transfer_variable_density_20To50}
    \end{subfigure}%
    \caption{Transferability of the proposed state-augmented RRM procedure, shown as the histogram of ergodic average rates, for the (a) fixed-density scenario, evaluated on samples with $m=200$, and (b) variable-density scenario, evaluated on samples with $m=50$.}
    \label{fig:transferability}
\end{figure*}

\begin{figure*}[h]
\centering
\includegraphics[width=\textwidth]{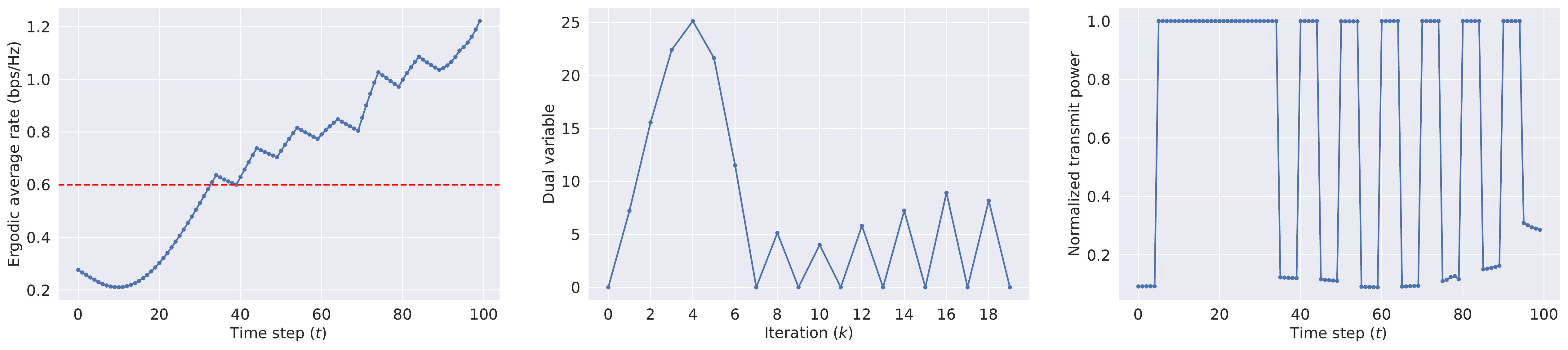}
\caption{Evolution of the ergodic average rate (left), dual variable (middle), and normalized transmit power level (right) for an example user in a network with $m=50$ transmitter-receiver pairs (for the variable-density scenario). The dashed line in the left plot represents the minimum-rate requirement, i.e., $f_{\min}=0.6$.}
\label{fig:policy_switching_m50_vardensity}
\end{figure*}

\begin{figure}[h]
\centering
\includegraphics[width=.95\columnwidth]{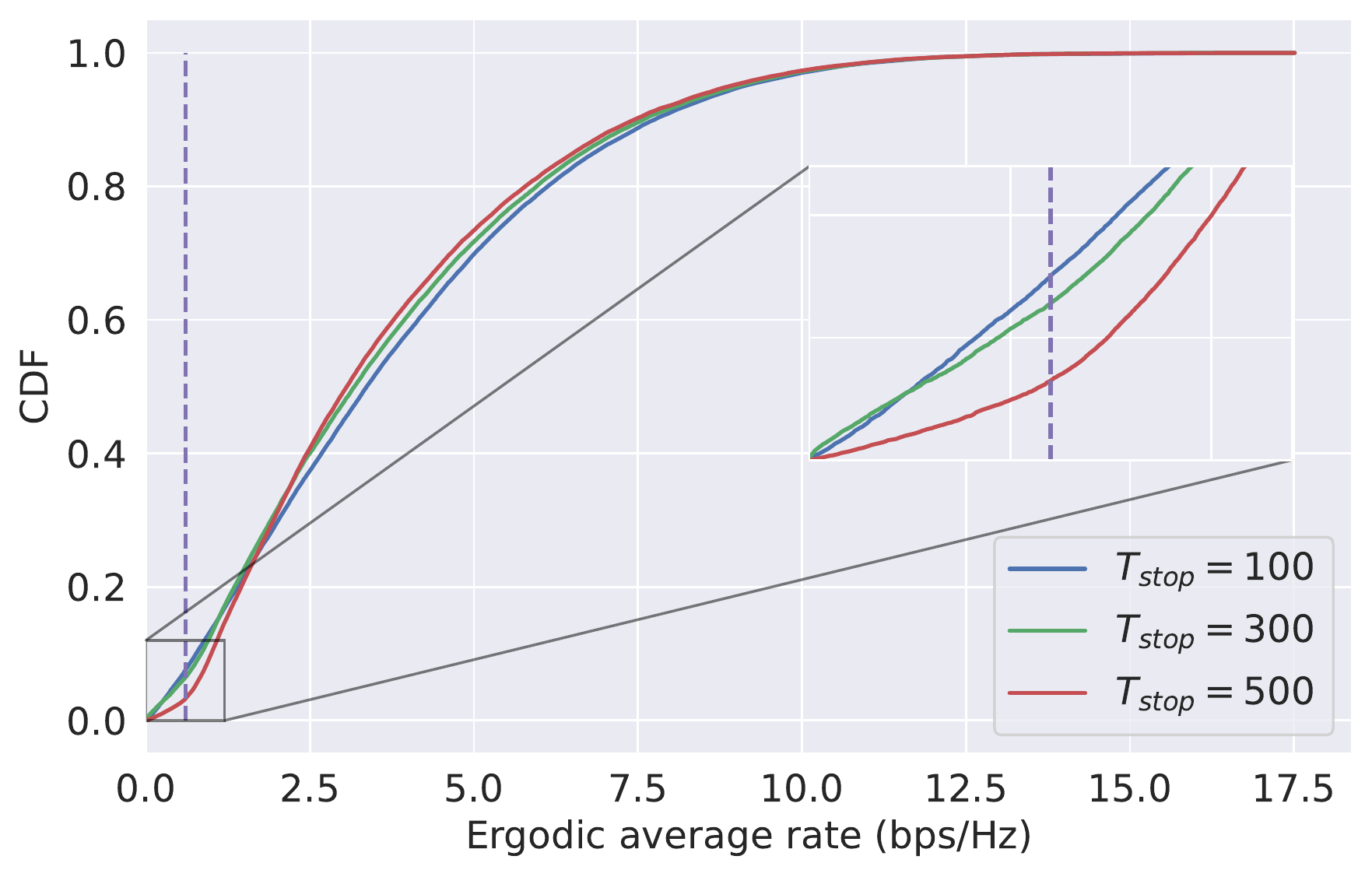}
\caption{Impact of the early stopping of the dual descent updates on users' ergodic average rates in networks with $m=50$ transmitter-receiver pairs (for the variable-density scenario). The dashed line represents the minimum-rate requirement, i.e., $f_{\min}=0.6$.}
\label{fig:early_stopping_m50_vardensity_T500}
\end{figure}

% \begin{figure*}[h]
% \centering
% \includegraphics[width=\textwidth]{figs/impactoffmin_m6_T0_5.pdf}
% \caption{Impact of the minimum-rate constraint lower bound, $f_{\min}$, on the performance achieved by the proposed state-augmented RRM algorithm in networks with $m=6$ transmitter-receiver pairs (for the variable-density scenario).}
% \label{fig:impactoffmin_m6_T0_5}
% \end{figure*}

\subsubsection{Single Network Realization}\label{sec:single_network}
Figure~\ref{fig:convergence_singleConfig_m50_vardensity} shows the performance of our proposed method against the baselines for a \emph{single} network realization with $m=50$ users (in the variable-density scenario). For this specific experiment, we continue training for $500$ epochs, and decay the primal learning rate by $0.5$ every $100$ epochs. At the end of each training epoch, we evaluate the RRM algorithm on the same realization that is used for training. As the figure shows, our proposed state-augmented RRM algorithm considerably gains over the baseline algorithms in terms of the minimum and 5\textsuperscript{th} percentile rates, managing to satisfy the minimum-rate requirements for all users at the expense of a smaller achieved mean rate.

\subsubsection{Scalability}
Here, we assess the performance of the the proposed algorithm in \emph{families} of networks with different numbers of users, where the policies trained on samples with a specific value of $m$ are evaluated on test samples with the same value of $m$. Figures~\ref{fig:scalability_fixed_density} and~\ref{fig:scalability_variable_density} compare the performance of our proposed method with the baseline algorithms in the fixed-density and variable-density scenarios, respectively. In both scenarios, thanks to the feasibility guarantees of our proposed method for the per-user minimum-rate constraints, our method significantly outperforms the baseline methods in terms of the minimum rate and the 5\textsuperscript{th} percentile rate. This, however, comes at the cost of a slightly lower mean rate. Note that such a scalability is due to the size-invariant property of GNN-based parameterizations (where the number of parameters is fixed regardless of the graph size), as opposed to other parameterizations, such as multi-layer perceptrons (i.e., fully-connected neural networks).

\subsubsection{Transferability to Unseen Network Sizes}
Another upside of the size-invariance property of GNNs is that they can be evaluated on graph sizes not encountered throughout the training process. Figure~\ref{fig:transfer_fixed_density_50To200} shows the histogram of user ergodic average rates under our proposed algorithm in test networks of size $m=200$ (fixed-density scenario) using two GNNs: one trained on samples with $m=200$ and the other trained on samples with $m=50$. Moreover, Figure~\ref{fig:transfer_variable_density_20To50} shows a similar histogram for the variable-density scenario, where GNNs trained on samples with $m=50$ and $m=20$ are evaluated on test samples with $m=50$. As the figures show, in both cases, the trained GNNs provide excellent transferability to larger network sizes, achieving user rates that are similar to the GNNs without the train/test mismatch. Moreover, as expected, the transferred performance is more desirable in the fixed-density scenario than the variable-density scenario, as the channel statistics mostly remain similar in the former case, while in the latter case, the network becomes more interference-limited as $m$ grows.

\subsubsection{Policy Switching}
Figure~\ref{fig:policy_switching_m50_vardensity} shows what an example receiver experiences over the course of the $T=100$ time steps in a test network with $m=50$ transmitter-receiver pairs (for the variable-density scenario) once training is completed. Letting $i$ denote the receiver index, at each time step $t$, we plot its ergodic average rate up to that time step (i.e., $\frac{1}{t} \sum_{\tau=1}^t f_i(\bbH_{\tau}, \bbp(\bbH_{\tau},\bbmu_{\lfloor \tau / T_0 \rfloor};\bbphi^{\star}))$), its dual variable at the corresponding iteration (i.e., $\mu_{i,\lfloor t / T_0 \rfloor}$), and the selected normalized transmit power of transmitter $i$ at that time step (i.e., $p_i(\bbH_t,\bbmu_{\lfloor t / T_0 \rfloor};\bbphi^{\star}) / P_{\max}$). As the figure shows, a \emph{policy switching} behavior occurs, where for the time steps in which the receiver's ergodic average rate is below $f_{\min}$, the dual variable increases, leading to the corresponding transmitter using full transmit power. Moreover, when the ergodic average rate exceeds $f_{\min}$, the transmitter uses an on-off pattern to maintain its receiver's high ergodic average rate, while minimizing the interference caused at other receivers. This shows why state-augmentation is crucial for the algorithm to be able to switch the policy if/when necessary, depending on the constraint violations over time. Indeed, neither the ``off'' policy ($p_i=0$) nor the ``on'' policy ($p_i=P_{\max}$) guarantees the satisfaction of the constraints by itself. However, the illustrated policy switching behavior in Figure~\ref{fig:policy_switching_m50_vardensity} is exactly what the transmitters need to show in order to be able to satisfy their constraints in the long run.

\subsubsection{Drawback of Early Stopping of Dual Descent Iterations}
As we mentioned in Section~\ref{sec:alg}, the feasibility and near-optimality of the RRM decisions depend on the fact that the dual descent iterations in~\eqref{eq:mu_dynamics_augmented} continue for the entire duration of the execution phase. This implies that such feasibility and near-optimality guarantees will not hold if the dual descent updates are stopped at a small, finite number of iterations. Figure~\ref{fig:early_stopping_m50_vardensity_T500} illustrates the empirical cumulative distribution function (CDF) of the user ergodic average rates in test networks with $m=50$ transmitter-receiver pairs (for the variable-density scenario), where the dual descent updates are stopped at a certain time step $T_{\text{stop}} \in \{100, 300, 500\}$. For the sake of this experiment, we expanded the execution time horizon to $T=500$ time steps. As the figure shows, early stopping of the dual variable updates leads to a larger fraction of users violating their minimum-rate requirements. Quite interestingly, such an early stopping provides an equivalent way of implementing the regular primal-dual REGNN method in~\cite{eisen2020optimal}.

\subsubsection{Inference Computational Complexity}
The size-invariance property of the GNNs implies that the number of GNN parameters does not depend on the underlying graph size on which the GNN is operating. However, using GNNs on larger graph sizes naturally leads to a higher computational complexity. Figure~\ref{fig:inference_time_vs_m} shows the average inference time (on CPU) for graphs corresponding to networks with $m\in[20, 200]$ transmitter-receiver pairs. Note that, in practice, the inference time can effectively be reduced by parallelizing computations on GPU hardware architectures.

\begin{figure}[t!]
\centering
\includegraphics[width=.97\columnwidth]{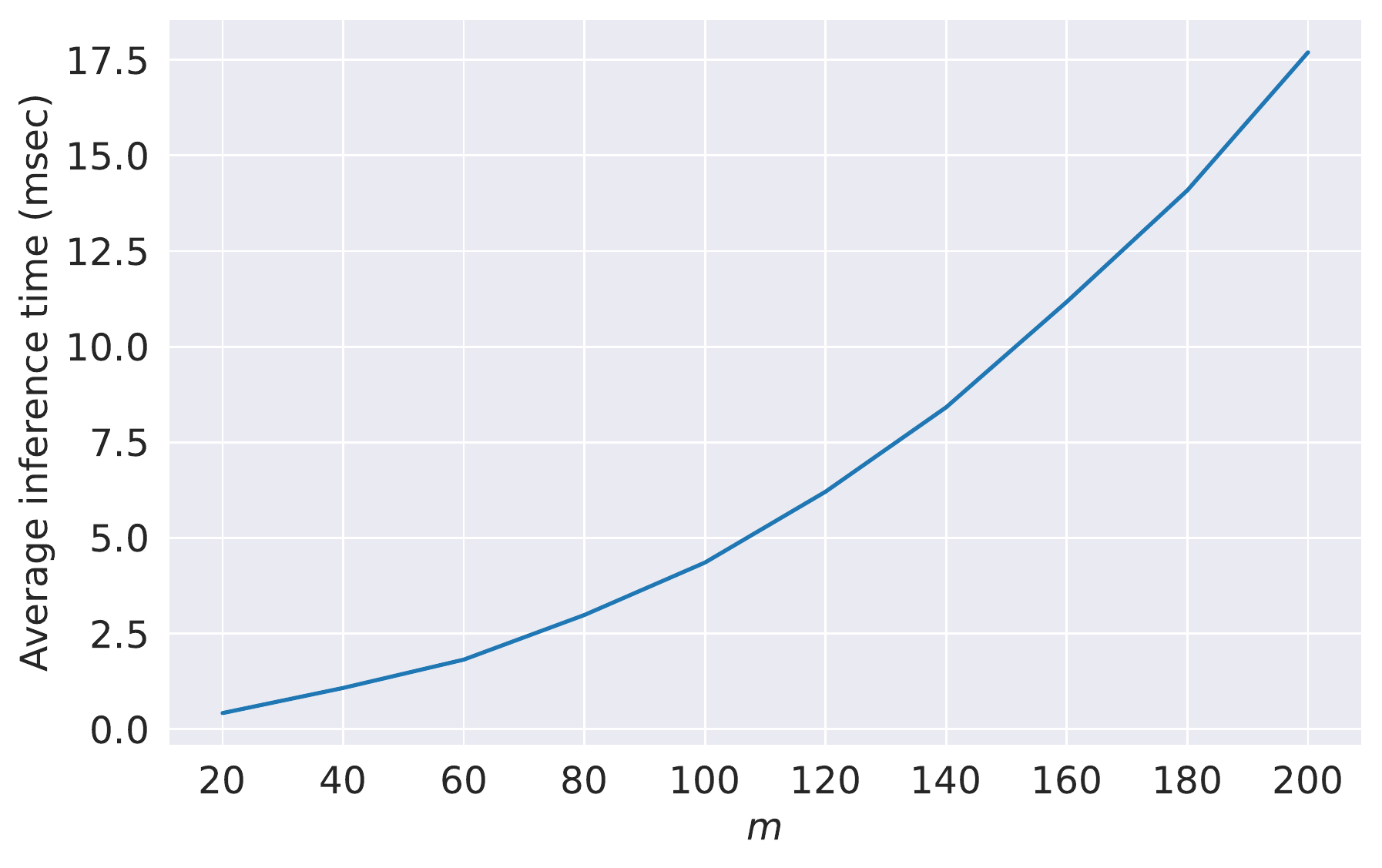}
\caption{Average inference time (on a CPU hardware architecture) for networks with $m\in[20,200]$ transmitter-receiver pairs.}
\label{fig:inference_time_vs_m}
\end{figure}

% \subsubsection{Impact of the Minimum-Rate Constraint Lower Bounds}
% In all our experiments, we set the lower bounds for the minimum-rate constraints, i.e., $f_{\min}$ to $\frac{3}{4}$. This resulted in RRM decisions that attempted to satisfy relatively strict minimum-rate constraints for almost all network realizations, which came at the cost of lower mean rates as compared to baseline methods. Figure~\ref{fig:impactoffmin_m6_T0_5} demonstrates the impact of the constraint lower bounds $f_{\min} \in \left\{\frac{1}{4}, \frac{1}{2}, \frac{3}{4}, 1\right\}$ on the performance of the proposed state-augmented RRM algorithm when trained and evaluated on networks with $m=6$ transmitter-receiver pairs (in the variable-density scenario). As the figure shows, tuning $f_{\min}$ reveals the trade-off between the average and worst-case performance of users across all network realizations. Finding the best minimum-rate constraint lower bound is a non-trivial problem, and it has also been shown that this bound can be tuned adaptively based on each specific network realization through the introduction of learnable slack parameters~\cite{naderializadeh2020wireless, naderializadeh2022learning, naderializadeh2022adaptive}. We leave the incorporation of adaptive minimum-rate constraints into state-augmented RRM algorithms as future work.

\section{Concluding Remarks}\label{sec:conclusion}
We considered the problem of resource allocation in wireless network, termed radio resource management (RRM), in which the goal is to maximize a network-wide utility subject to constraints on the long-term average performance of the network users. We showed how parameterized dual-based RRM policies can lead to feasible and near-optimal solutions, but suffer from multiple challenges, including that they need to be run for an infinite number of time steps and the model parameters have to be optimized for any given set of dual variables at every time step. To mitigate these challenges, we proposed a state-augmented RRM algorithm, which revises the parameterization to take as input the network state, as well as the dual variable at each step, and produce as output the RRM decisions. We showed that for near-universal parameterizations, the proposed method also leads to feasible and near-optimal sequences of RRM decisions. Using graph neural network (GNN) architectures to parameterize the state-augmented RRM policy, we also demonstrated, via a set of numerical experiments, that the proposed method provides scalable and transferable wireless power control solutions that outperform baseline algorithms.

\appendices
\section{Proof of Theorem~\ref{thm:main}}\label{appx:proof}
\allowdisplaybreaks

The arguments used in the proof of Theorem~\ref{thm:main} follow similar steps to those in~\cite{calvo2021state}, with some minor differences specific to our problem formulation, as outlined in Remark~\ref{remark:proof_diff}. We, therefore, include the complete proof here for the paper to be self-complete.% there are differences with prior proofs, specific to our problem formulation, that require us to include the complete proof here.

\subsection{Tightness of The Sequence of Dual Variable Probability Measures}\label{appx:proof_tightness}
We start by proving that the sequence of dual variable probability measures, i.e., $\{p(\bbmu_k|\bbmu_0)\}_{k}$, is tight in the following lemma.

\begin{lemma}
For any $\delta > 0$, there exists a compact set $\ccalA_{\delta}$ such that for every $k \geq 0$, we have $\mathsf{Pr}[\bbmu_k \in \ccalA_{\delta}] > 1 - \delta$.
\end{lemma}

\begin{proof}
Define the dual function, $d(\bbmu)$ as
\begin{align}\label{eq:def_dual_function}
d(\bbmu) = \max_{\bbtheta\in\bbTheta} \ccalL(\bbtheta, \bbmu),
\end{align}
and consider the following set:
\begin{align}\label{eq:def_C}
\ccalC = \left\{\bbmu \in \reals_+^c : d(\bbmu) - P^{\star} \leq \frac{c\eta_{\bbmu}G^2}{2} \right\}.
\end{align}
Now, let
\begin{align}
\ccalA_{\delta} = \ccalB_{\delta} \cup \left(\ccalC \oplus \ccalS_{\sqrt{c\eta_{\bbmu}^2 G^2}}\right),
\end{align}
where $\oplus$ denotes Minkowski sum, $\ccalS_r$ denotes a ball centered at the origin with radius $r\geq0$, and $\ccalB_{\delta}$ is defined as
\begin{align}
\ccalB_{\delta} = \left\{\bbmu \in \reals_+^c : \frac{\| \bbmu - \bbmu^{\star} \|}{\| \bbmu_0 - \bbmu^{\star}\|} \leq \frac{1}{\delta} \right\},
\end{align}
with $\bbmu^{\star}$ representing the optimal set of dual variables. Then, we have
\begin{align}
&\mathsf{Pr}[\bbmu_k \in \ccalA_{\delta}] \nonumber \\
& = \mathsf{Pr}\left[\bbmu_k \in \ccalB_{\delta} \cup \left(\ccalC \oplus \ccalS_{\sqrt{c\eta_{\bbmu}^2 G^2}}\right) \right] \\
& = \underbrace{\mathsf{Pr}\left[\bbmu_k \in \ccalB_{\delta} \cup \left(\ccalC \oplus \ccalS_{\sqrt{c\eta_{\bbmu}^2 G^2}}\right) \middle| \bbmu_{k-1}\in\ccalC \right]}_{\overset{(a)}{=}1} \cdot \ p \nonumber \\
& \quad + \underbrace{\mathsf{Pr}\left[\bbmu_k \in \ccalB_{\delta} \cup \left(\ccalC \oplus \ccalS_{\sqrt{c\eta_{\bbmu}^2 G^2}}\right) \middle| \bbmu_{k-1}\in\ccalC^c \right]}_{\geq \mathsf{Pr}\left[\bbmu_k \in \ccalB_{\delta}  \middle| \bbmu_{k-1}\in\ccalC^c \right]} \cdot \ (1-p) \nonumber \\
& \geq \mathsf{Pr}\left[\bbmu_k \in \ccalB_{\delta}  \middle| \bbmu_{k-1}\in\ccalC^c \right] \\
& = \mathsf{Pr}\left[\frac{\| \bbmu_k - \bbmu^{\star} \|}{\| \bbmu_0 - \bbmu^{\star}\|} \leq \frac{1}{\delta}  \middle| \bbmu_{k-1}\in\ccalC^c \right] \\
& \geq 1 - \left( \frac{\E\left[\| \bbmu_k - \bbmu^{\star} \| \middle| \bbmu_{k-1}\in\ccalC^c\right]}{\| \bbmu_0 - \bbmu^{\star}\|} \right) \delta,
\end{align}
where $p=\mathsf{Pr}\left[\bbmu_{k-1}\in\ccalC \right]$, (a) is true due to the dual dynamics in~\eqref{eq:mu_dynamics} and the fact that the change between $\bbmu_{k-1}$ and $\bbmu_k$ is upper bounded in norm by $\sqrt{c\eta_{\bbmu}^2 G^2}$, and the last inequality is due to conditional Markov's inequality. Therefore, to complete the proof, it suffices to show that (i)
\begin{align}\label{eq:bounded_conditional_difference_norm_muk}
\E\left[\| \bbmu_k - \bbmu^{\star} \| \middle| \bbmu_{k-1}\in\ccalC^c\right] < \| \bbmu_0 - \bbmu^{\star}\|,
\end{align}
and (ii) that $\ccalA_{\delta}$ is compact.

\noindent\textbf{Proof of~\eqref{eq:bounded_conditional_difference_norm_muk}.}
Let $\Delta_{\bbmu_{k-1}}$ be defined as
\begin{align}
\Delta_{\bbmu_{k-1}} \coloneqq \bbg\left( \frac{1}{T_0} \sum_{t=(k-1)T_0}^{kT_0-1} \bbf(\bbH_{t}, \bbp(\bbH_{t};\bbtheta_{k-1})) \right).
\end{align}
Then, from the dual dynamics in~\eqref{eq:mu_dynamics}, we have
\begin{align*}
&\|\bbmu_{k} - \bbmu^{\star}\|^2 \\
&\leq \|\bbmu_{k-1} - \bbmu^{\star} - \eta_{\bbmu} \Delta_{\bbmu_{k-1}} \|^2 \\
&= \|\bbmu_{k-1} - \bbmu^{\star}  \|^2 + \eta_{\bbmu}^2 \|\Delta_{\bbmu_{k-1}} \|^2 \nonumber \\
&\quad - 2\eta_{\bbmu} (\bbmu_{k-1} - \bbmu^{\star})^T \Delta_{\bbmu_{k-1}} \\
&\leq \|\bbmu_{k-1} - \bbmu^{\star}  \|^2 + c \eta_{\bbmu}^2 G^2 - 2\eta_{\bbmu} (\bbmu_{k-1} - \bbmu^{\star})^T \Delta_{\bbmu_{k-1}},
\end{align*}
where the last inequality follows from the fact that for any $i\in\{1,\dots,c\}$, we assume $|g_i\left(\cdot\right)|\leq G$. Taking the expectation of both sides conditioned on $\bbmu_{k-1}$, we have
\begin{align}
&\E\left[\|\bbmu_{k} - \bbmu^{\star}\|^2 \middle| \bbmu_{k-1}\right] \nonumber\\
&\leq \|\bbmu_{k-1} - \bbmu^{\star}  \|^2 + c \eta_{\bbmu}^2 G^2 \nonumber \\
&\quad - 2\eta_{\bbmu} (\bbmu_{k-1} - \bbmu^{\star})^T \E\left[\Delta_{\bbmu_{k-1}}\middle| \bbmu_{k-1}\right]\\
&\leq \|\bbmu_{k-1} - \bbmu^{\star}  \|^2 + c \eta_{\bbmu}^2 G^2 - 2\eta_{\bbmu} (d(\bbmu_{k-1}) - d(\bbmu^{\star}))\label{eq:subgradient}\\
&= \|\bbmu_{k-1} - \bbmu^{\star}  \|^2 + c \eta_{\bbmu}^2 G^2 - 2\eta_{\bbmu} (d(\bbmu_{k-1}) - P^{\star}),\label{eq:pre_C_bound}
\end{align}
where~\eqref{eq:subgradient} follows from the fact that $\Delta_{\bbmu_{k-1}}$ is a subgradient of the dual function~\cite{danskin2012theory}, leading to the inequality $(\bbmu_{k-1} - \bbmu^{\star})^T \E\left[\Delta_{\bbmu_{k-1}}\middle| \bbmu_{k-1}\right] \geq d(\bbmu_{k-1}) - d(\bbmu^{\star})$ thanks to the dual function being convex~\cite{boyd2004convex}. Now, combining the fact that $\bbmu_{k-1}\in\ccalC^c$ with the definition of $\ccalC$ in~\eqref{eq:def_C}, we can continue~\eqref{eq:pre_C_bound} as
\begin{align}
\E\left[\|\bbmu_{k} - \bbmu^{\star}\|^2 \middle| \bbmu_{k-1}\in\ccalC^c\right]
&< \|\bbmu_{k-1} - \bbmu^{\star}\|^2 \\
&< \|\bbmu_{0} - \bbmu^{\star}\|^2,
\end{align}
with the last inequality following from recursion.

\noindent\textbf{Proof of Compactness of $\ccalA_{\delta}$.} Since the Minkowski sum of two compact sets is compact, and also the union of two compact sets is compact, it suffices to show that the set $\ccalC$ is compact.

Given the definition of the dual function in~\eqref{eq:def_dual_function}, for any set of model parameters $\bbtheta\in\bbTheta$ and any set of dual variables $\bbmu\in\reals_+^c$, we have
\begin{align}
d(\bbmu) &\geq \mathcal{U}\left( \frac{1}{T} \sum_{t=0}^{T-1} \bbf(\bbH_t, \bbp(\bbH_t;\bbtheta)) \right)\nonumber \\
   &\qquad\qquad+ \bbmu^T \bbg\left( \frac{1}{T} \sum_{t=0}^{T-1} \bbf(\bbH_t, \bbp(\bbH_t;\bbtheta)) \right)
\end{align}
Now, replacing $\bbtheta$ with the strictly-feasible set of model parameters $\hat{\bbtheta}$ considered in Theorem~\ref{thm:main}, we can write
\begin{align}
d(\bbmu) &\geq \mathcal{U}\left( \frac{1}{T} \sum_{t=0}^{T-1} \bbf(\bbH_t, \bbp(\bbH_t;\hat{\bbtheta})) \right)\nonumber \\
   &\qquad\qquad+ \bbmu^T \bbg\left( \frac{1}{T} \sum_{t=0}^{T-1} \bbf(\bbH_t, \bbp(\bbH_t;\hat{\bbtheta})) \right)\\
&\geq \mathcal{U}\left( \frac{1}{T} \sum_{t=0}^{T-1} \bbf(\bbH_t, \bbp(\bbH_t;\hat{\bbtheta})) \right) + G' \| \bbmu \|_1.\label{eq:dual_upper_bound}
\end{align}
Now, for any $\bbmu\in\ccalC$, according to the definition in~\eqref{eq:def_C}, the bound in~\eqref{eq:dual_upper_bound} leads to
\begin{align}
P^{\star} + \frac{c\eta_{\bbmu}G^2}{2}  \geq \mathcal{U}\left( \frac{1}{T} \sum_{t=0}^{T-1} \bbf(\bbH_t, \bbp(\bbH_t;\hat{\bbtheta})) \right) + G' \| \bbmu \|_1,
\end{align}
implying that $\bbmu$ belongs to the following ball,
\begin{align}
\| \bbmu \|_1 \leq \frac{P^{\star} + \frac{c\eta_{\bbmu}G^2}{2} - \mathcal{U}\left( \frac{1}{T} \sum_{t=0}^{T-1} \bbf(\bbH_t, \bbp(\bbH_t;\hat{\bbtheta})) \right)}{G'},
\end{align}
hence completing the proof.

\end{proof}

\subsection{Proof of Feasibility in~\eqref{eq:thm_feasibility}}\label{appx:proof_feasibility}

From the dual dynamics in~\eqref{eq:mu_dynamics}, we have
\begin{align}\label{eq:mu_inequality}
% \bbmu_{T+1} &\geq \bbmu_T - \eta_{\bbmu} \bbg\left(  \bbf(\bbH_T, \bbp(\bbH_T;\bbtheta_T)) \right).
\bbmu_{K} &\geq \bbmu_{K-1} - \eta_{\bbmu} \bbg\left( \frac{1}{T_0} \sum_{t=(K-1)T_0}^{KT_0-1} \bbf(\bbH_{t}, \bbp(\bbH_{t};\bbtheta_{K-1})) \right)\hspace{-3pt}.
\end{align}
Using the inequality in~\eqref{eq:mu_inequality} recursively yields
% \begin{align}
% \bbmu_{T+1} &\geq \bbmu_1 - \eta_{\bbmu} \sum_{t=1}^T \bbg\left(  \bbf(\bbH_t, \bbp(\bbH_t;\bbtheta_t)) \right) \\
% &= \bbmu_1 - T \eta_{\bbmu} \left(\frac{1}{T} \sum_{t=1}^T  \bbg\left(  \bbf(\bbH_t, \bbp(\bbH_t;\bbtheta_t)) \right)\right) \\
% &\geq \bbmu_1 - T \eta_{\bbmu}  \bbg\left( \frac{1}{T} \sum_{t=1}^T  \bbf(\bbH_t, \bbp(\bbH_t;\bbtheta_t)) \right),
% \end{align}
\begin{align*}
&\bbmu_{K} \nonumber \\
&\geq \bbmu_0 - \eta_{\bbmu} \sum_{k=0}^{K-1} \bbg\left( \frac{1}{T_0} \sum_{t=kT_0}^{(k+1)T_0-1} \bbf(\bbH_{t}, \bbp(\bbH_{t};\bbtheta_k)) \right) \nonumber \\
&= \bbmu_0 - K\eta_{\bbmu} \cdot \frac{1}{K}\hspace{-3pt}\sum_{k=0}^{K-1} \bbg\left( \frac{1}{T_0} \sum_{t=kT_0}^{(k+1)T_0-1} \bbf(\bbH_{t}, \bbp(\bbH_{t};\bbtheta_k)) \right) \nonumber \\
&\geq \bbmu_0 - K\eta_{\bbmu} \bbg\left( \frac{1}{KT_0} \sum_{t=0}^{KT_0-1} \bbf(\bbH_t, \bbp(\bbH_t;\bbtheta_{\lfloor t/T_0 \rfloor})) \right),
\end{align*}
where the last inequality follows from the concavity of $\bbg(\cdot)$. Therefore, we have
% \begin{align}\label{eq:muT+1_vs_mu1}
% &\limsup_{T\to\infty}\bbmu_{T+1}  \nonumber\\ 
% &\geq \bbmu_1 - \eta_{\bbmu} \liminf_{T\to\infty} T \bbg\left( \frac{1}{T} \sum_{t=1}^T  \bbf(\bbH_t, \bbp(\bbH_t;\bbtheta_t)) \right).
% \end{align}
\begin{align}\label{eq:muT+1_vs_mu1}
&\limsup_{K\to\infty}\bbmu_{K+1}  \nonumber\\ 
&\geq \bbmu_0 - \eta_{\bbmu} \liminf_{K\to\infty} K \bbg\left( \frac{1}{T} \sum_{t=0}^{T-1} \bbf(\bbH_t, \bbp(\bbH_t;\bbtheta_{\lfloor t/T_0 \rfloor})) \right).
\end{align}
Now, assume by contradiction that one of the constraints in~\eqref{eq:thm_feasibility} is not satisfied. More precisely, assume there exists an index $i\in\{1,\dots,c\}$ and positive constants $\delta>0$ and $\beta\in(0,1)$ such that
\begin{align}\label{eq:contradiction_prob}
\mathsf{Pr}\left[\liminf_{T\to\infty} g_i\left( \frac{1}{T} \sum_{t=0}^{T-1} \bbf(\bbH_t, \bbp(\bbH_t;\bbtheta_{\lfloor t/T_0 \rfloor})) \right) \leq -\delta \right] = \beta.
\end{align}
This implies that with non-zero probability $\beta$, we would have
% \begin{align}
% &\limsup_{T\to\infty} \|\bbmu_{T+1}\|_1 \nonumber \\
% &\geq \limsup_{T\to\infty} \mu_{i,T+1}\label{eq:1_norm} \\
% &\geq \mu_{i,1} - \eta_{\bbmu} \liminf_{T\to\infty} T g_i\left( \frac{1}{T} \sum_{t=1}^T  \bbf(\bbH_t, \bbp(\bbH_t;\bbtheta_t)) \right)\label{eq:ith_bound} \\
% &\geq \mu_{i,1} + \eta_{\bbmu} \liminf_{T\to\infty} T \epsilon \label{eq:contradiction_prob_use} \\
% &= \infty,
% \end{align}
\begin{align}
&\limsup_{K\to\infty} \|\bbmu_{K}\|_1 \nonumber \\
&\geq \limsup_{K\to\infty} \mu_{i,K}\label{eq:1_norm} \\
&\geq \mu_{i,0} - \eta_{\bbmu} \liminf_{K\to\infty} K g_i\left( \frac{1}{T} \sum_{t=0}^{T-1} \bbf(\bbH_t, \bbp(\bbH_t;\bbtheta_{\lfloor t/T_0 \rfloor})) \right)\label{eq:ith_bound} \\
&\geq \mu_{i,0} + \eta_{\bbmu} \liminf_{K\to\infty} K \epsilon \label{eq:contradiction_prob_use} \\
&= \infty,
\end{align}
where~\eqref{eq:1_norm} follows from the definition of the $\ell_1$-norm% and the fact that $\mu_{i,T+1} \geq 0$
,~\eqref{eq:ith_bound} follows from the $i$\textsuperscript{th} inequality in~\eqref{eq:muT+1_vs_mu1}, and~\eqref{eq:contradiction_prob_use} follows from~\eqref{eq:contradiction_prob}. This contradicts the fact that the sequence of dual variable probabilities is tight, hence completing the proof.
\hfill$\square$

\subsection{Proof of Optimality in~\eqref{eq:thm_optimality}}

Given the dual dynamics in~\eqref{eq:mu_dynamics} and the fact that projection onto the non-negative orthant does not increase the $\ell_2$-norm, we have
\begin{align}
&\left\|\bbmu_{K+1}\right\|^2\nonumber\\
&\leq \left\|\bbmu_K - \eta_{\bbmu} \bbg\left( \frac{1}{T_0} \sum_{t=KT_0}^{(K+1)T_0-1} \bbf(\bbH_{t}, \bbp(\bbH_{t};\bbtheta_K)) \right) \right\|^2\label{eq:norm_dual_update}\\
&=\left\|\bbmu_K \right\|^2 + \eta_{\bbmu}^2 \left\| \bbg\left( \frac{1}{T_0} \sum_{t=KT_0}^{(K+1)T_0-1} \bbf(\bbH_{t}, \bbp(\bbH_{t};\bbtheta_K)) \right)\right\|^2 \nonumber \\
&~\quad- 2\eta_{\bbmu}\bbmu_K^T \bbg\left( \frac{1}{T_0} \sum_{t=KT_0}^{(K+1)T_0-1} \bbf(\bbH_{t}, \bbp(\bbH_{t};\bbtheta_K)) \right)\label{eq:norm_dual_square_expansion}\\
&\leq\left\|\bbmu_T\right\|^2 + c\eta_{\bbmu}^2 G^2  \nonumber \\
&~\quad - 2\eta_{\bbmu}\bbmu_K^T \bbg\left( \frac{1}{T_0} \sum_{t=KT_0}^{(K+1)T_0-1} \bbf(\bbH_{t}, \bbp(\bbH_{t};\bbtheta_K)) \right),\label{eq:g_bounded}
\end{align}
% Since $\left\|\bbmu_{T+1}\right\|^2 \geq 0$, expanding the squared norm on the right-hand side of~\eqref{eq:norm_dual_update} yields
% \begin{align}\label{eq:norm_dual_square_expansion}
% 0 &\leq \left\|\bbmu_T\right\|^2 + \eta_{\bbmu}^2 \left\| \bbg\left(  \bbf(\bbH_T, \bbp(\bbH_T, \bbmu_T;\bbtheta)) \right)\right\|^2 - 2\eta_{\bbmu}\bbmu_T^T \bbg\left(  \bbf(\bbH_T, \bbp(\bbH_T, \bbmu_T;\bbtheta)) \right) .
% \end{align}
%Since 
where the last inequality follows from the fact that for any $i\in\{1,\dots,c\}$, we assume $|g_i\left(\cdot\right)|\leq G$. Applying~\eqref{eq:g_bounded} recursively yields
\begin{align}
&\left\|\bbmu_{K+1}\right\|^2 \nonumber\\
&\leq \left\|\bbmu_0\right\|^2 + c\eta_{\bbmu}^2 K G^2  \nonumber \\
&~\quad - 2\eta_{\bbmu} \sum_{k=0}^{K-1} \bbmu_k^T \bbg\left( \frac{1}{T_0} \sum_{t=kT_0}^{(k+1)T_0-1} \bbf(\bbH_{t}, \bbp(\bbH_{t};\bbtheta_k)) \right).\label{eq:g_bounded_recursive}
\end{align}
Since $\left\|\bbmu_{K+1}\right\|^2 \geq 0$, rearranging the terms in~\eqref{eq:g_bounded_recursive} and normalizing both sides by $2\eta_{\bbmu}K$ results in
% \begin{align}\label{eq:inner_prod_dual_init}
% \bbmu_T^T \bbg\left(  \bbf(\bbH_T, \bbp(\bbH_T, \bbmu_T;\bbtheta)) \right) \leq \frac{1}{2\eta_{\bbmu}}\left\|\bbmu_T\right\|^2 + \frac{c\eta_{\bbmu}G^2}{2}.
% \end{align}
% Applying~\eqref{eq:inner_prod_dual_init} recursively and normalizing by $T$, we will have
\begin{align}\label{eq:inner_prod_dual_recursive_normalized}
&\frac{1}{K} \sum_{k=0}^{K-1} \bbmu_k^T \bbg\left( \frac{1}{T_0} \sum_{t=kT_0}^{(k+1)T_0-1} \bbf(\bbH_{t}, \bbp(\bbH_{t};\bbtheta_k)) \right)\nonumber \\
&\leq \frac{1}{2\eta_{\bbmu}K} \left\|\bbmu_0\right\|^2 + \frac{c\eta_{\bbmu}G^2}{2}.
\end{align}
Taking the conditional expectation of both sides in~\eqref{eq:inner_prod_dual_recursive_normalized} given $\bbmu_0$, and letting $K\to\infty$, we can write
\begin{align}
&\limsup_{K\to\infty} \E\left[\sum_{k=0}^{K-1} \frac{\bbmu_k^T}{K} \bbg\left(  \sum_{t=kT_0}^{(k+1)T_0-1} \frac{\bbf(\bbH_{t}, \bbp(\bbH_{t};\bbtheta_k))}{T_0} \right)\middle|\bbmu_0\right] \nonumber \\
&\leq \frac{c\eta_{\bbmu}G^2}{2}.\label{eq:inner_prod_dual_limsup}
\end{align}

For any iteration $k$, since $\bbtheta_k$ is the maximizer of~\eqref{eq:theta_dynamics}, for any $\bbtheta\in\bbTheta$, we can write
\begin{align}
& \mathcal{U}\left( \frac{1}{T_0} \sum_{t=kT_0}^{(k+1)T_0-1} \bbf(\bbH_{t}, \bbp(\bbH_{t};\bbtheta_k)) \right)\nonumber \\
   &\qquad+ \bbmu_{k}^T \bbg\left( \frac{1}{T_0} \sum_{t=kT_0}^{(k+1)T_0-1} \bbf(\bbH_{t}, \bbp(\bbH_{t};\bbtheta_k)) \right)\nonumber   \\
& \geq \mathcal{U}\left( \frac{1}{T_0} \sum_{t=kT_0}^{(k+1)T_0-1} \bbf(\bbH_{t}, \bbp(\bbH_{t};\bbtheta)) \right)\nonumber \\
   &\qquad+ \bbmu_{k}^T \bbg\left( \frac{1}{T_0} \sum_{t=kT_0}^{(k+1)T_0-1} \bbf(\bbH_{t}, \bbp(\bbH_{t};\bbtheta)) \right)\nonumber 
\end{align}

Taking the expectations of both sides conditioned on $\bbmu_{k}$, we have
\begin{align}
&\E\left[\mathcal{U}\left( \frac{1}{T_0} \sum_{t=kT_0}^{(k+1)T_0-1} \bbf(\bbH_{t}, \bbp(\bbH_{t};\bbtheta_k)) \right)\middle| \bbmu_{k} \right]\nonumber \\
   &\qquad+ \bbmu_{k}^T \E\left[\bbg\left( \frac{1}{T_0} \sum_{t=kT_0}^{(k+1)T_0-1} \bbf(\bbH_{t}, \bbp(\bbH_{t};\bbtheta_k)) \right)\middle| \bbmu_{k} \right]\nonumber   \\
&\geq \E\left[\mathcal{U}\left( \frac{1}{T_0} \sum_{t=kT_0}^{(k+1)T_0-1} \bbf(\bbH_{t}, \bbp(\bbH_{t};\bbtheta)) \right)\right]\nonumber \\
   &\qquad+ \bbmu_{k}^T \E\left[\bbg\left( \frac{1}{T_0} \sum_{t=kT_0}^{(k+1)T_0-1} \bbf(\bbH_{t}, \bbp(\bbH_{t};\bbtheta)) \right)\right] \\
&= \E\left[\mathcal{U}\left( \frac{1}{T_0} \sum_{t=kT_0}^{(k+1)T_0-1} \bbf(\bbH_{t}, \bbp(\bbH_{t};\bbtheta)) \right)\right]\nonumber \\
   &\qquad+ \underbrace{\bbmu_{k}^T}_{\geq 0} \underbrace{\E\left[\bbg\left( \frac{1}{T_0} \sum_{t=kT_0}^{(k+1)T_0-1} \bbf(\bbH_{t}, \bbp(\bbH_{t};\bbtheta)) \right)\right]}_{\geq 0 \text{ for any feasible }\bbtheta\in\bbTheta}\label{eq:unbiased_estimate_prev_exp}\\
&= \lim_{T\to\infty} \mathcal{U}\left( \frac{1}{T} \sum_{t=0}^{T-1} \bbf(\bbH_t, \bbp(\bbH_t;\bbtheta)) \right)\nonumber \\
   &\qquad+ \underbrace{\bbmu_{k}^T}_{\geq 0} \lim_{T\to\infty} \underbrace{\bbg\left( \frac{1}{T} \sum_{t=0}^{T-1} \bbf(\bbH_t, \bbp(\bbH_t;\bbtheta)) \right)}_{\geq 0 \text{ for any feasible }\bbtheta\in\bbTheta}\label{eq:unbiased_estimate_Ug_Tinf}\\
& \geq \lim_{T\to\infty} \mathcal{U}\left( \frac{1}{T} \sum_{t=0}^{T-1} \bbf(\bbH_t, \bbp(\bbH_t;\bbtheta)) \right), \label{eq:objective_RHS}
\end{align}
where~\eqref{eq:unbiased_estimate_Ug_Tinf} follows from the fact that the expected value of the utility and the constraints within each iteration provide unbiased estimates of the objective and constraints in~\eqref{eq:param_problem}, respectively; i.e., for any $k\in\{0,1,2,\dots,K-1\}$ and $\forall \bbtheta \in \bbTheta$,
\begin{align}
&\E\left[\mathcal{U}\left( \frac{1}{T_0} \sum_{t=kT_0}^{(k+1)T_0-1} \bbf(\bbH_{t}, \bbp(\bbH_{t};\bbtheta)) \right)\right] \nonumber \\
&\qquad\qquad\qquad = \lim_{T\to\infty}\mathcal{U}\left( \frac{1}{T} \sum_{t=0}^{T-1} \bbf(\bbH_t, \bbp(\bbH_t;\bbtheta)) \right) \\
&\E\left[\bbg\left( \frac{1}{T_0} \sum_{t=kT_0}^{(k+1)T_0-1} \bbf(\bbH_{t}, \bbp(\bbH_{t};\bbtheta)) \right)\right] \nonumber \\
&\qquad\qquad\qquad = \lim_{T\to\infty}\bbg\left( \frac{1}{T} \sum_{t=0}^{T-1} \bbf(\bbH_t, \bbp(\bbH_t;\bbtheta)) \right).
\end{align}
The inequality in~\eqref{eq:objective_RHS} is true for all feasible $\bbtheta\in\bbTheta$, especially for the optimal set of parameters $\bbtheta^{\star}$, for which the RHS of~\eqref{eq:objective_RHS} equals $P^{\star}$. Therefore, we have
\begin{align}
&\E\left[\mathcal{U}\left( \frac{1}{T_0} \sum_{t=kT_0}^{(k+1)T_0-1} \bbf(\bbH_{t}, \bbp(\bbH_{t};\bbtheta_k)) \right)\middle| \bbmu_{k} \right]\nonumber \\
&\geq P^{\star} - \bbmu_{k}^T \E\left[\bbg\left( \frac{1}{T_0} \sum_{t=kT_0}^{(k+1)T_0-1} \bbf(\bbH_{t}, \bbp(\bbH_{t};\bbtheta_k)) \right)\middle| \bbmu_{k} \right].
% &\geq P^{\star} - \lim_{T\to\infty} \bbmu_t^T \bbg\left( \frac{1}{T} \sum_{t'=t}^{t+T-1} \E\left[\bbf(\bbH_{t'}, \bbp(\bbH_{t'};\bbtheta_t)) | \bbmu_t\right] \right) \\
% &\geq P^{\star} - \lim_{T\to\infty} \bbmu_t^T \bbg\left( \frac{1}{T} \sum_{t'=KT}^{t+T-1} \E\left[\bbf(\bbH_{t'}, \bbp(\bbH_{t'};\bbtheta_t)) | \bbmu_t\right] \right) \\
% &= P^{\star} - \lim_{T\to\infty} \bbmu_t^T \bbg\left( \E\left[\bbf(\bbH_{t}, \bbp(\bbH_{t};\bbtheta_t))| \bbmu_t\right] \right) \nn{?}
\end{align}
Averaging the above inequality over $k\in\{0,\dots,K-1\}$, and taking the expected value conditioned on $\bbmu_0$, we will get
\begin{align}
&\E\Bigg[\frac{1}{K} \sum_{k=0}^{K-1} \mathcal{U}\left( \frac{1}{T_0} \sum_{t=kT_0}^{(k+1)T_0-1} \bbf(\bbH_{t}, \bbp(\bbH_{t};\bbtheta_k)) \right)\Bigg| \bbmu_{0} \Bigg]\nonumber \\
&\geq P^{\star} \nonumber \\
& -\E\Bigg[\frac{1}{K} \sum_{k=0}^{K-1} \bbmu_{k}^T \bbg\left( \frac{1}{T_0} \sum_{t=kT_0}^{(k+1)T_0-1} \bbf(\bbH_{t}, \bbp(\bbH_{t};\bbtheta_k)) \right)\Bigg| \bbmu_{0} \Bigg]\hspace{-1pt}.\label{eq:avg_K_ineq}
\end{align}
Letting $K\to\infty$, and plugging~\eqref{eq:inner_prod_dual_limsup} into~\eqref{eq:avg_K_ineq}, we have
\begin{align}
& P^{\star} - \frac{c\eta_{\mu}G^2}{2} \nonumber \\
&\leq P^{\star} \nonumber \\
& - \limsup_{K\to\infty} \E\left[\sum_{k=0}^{K-1} \frac{\bbmu_k^T}{K} \bbg\left(  \sum_{t=kT_0}^{(k+1)T_0-1} \tfrac{\bbf(\bbH_{t}, \bbp(\bbH_{t};\bbtheta_k))}{T_0} \right)\middle|\bbmu_0\right] \\
&\leq \liminf_{K\to\infty} \E\Bigg[\frac{1}{K}\hspace{-2pt} \sum_{k=0}^{K-1} \mathcal{U}\left( \frac{1}{T_0} \sum_{t=kT_0}^{(k+1)T_0-1} \bbf(\bbH_{t}, \bbp(\bbH_{t};\bbtheta_k)) \right) \Bigg] \\
&\leq \liminf_{K\to\infty} \E\Bigg[ \mathcal{U}\left( \frac{1}{K}\hspace{-2pt} \sum_{k=0}^{K-1} \frac{1}{T_0} \sum_{t=kT_0}^{(k+1)T_0-1} \bbf(\bbH_{t}, \bbp(\bbH_{t};\bbtheta_k)) \right) \Bigg] \label{eq:U_concavity}\\
&= \liminf_{T\to\infty} \E\Bigg[ \mathcal{U}\left( \frac{1}{T} \sum_{t=0}^{T-1} \bbf(\bbH_{t}, \bbp(\bbH_{t};\bbtheta_{\lfloor t / T_0 \rfloor})) \right) \Bigg] \label{eq:final_optimality}
\end{align}
where~\eqref{eq:U_concavity} is due to concavity of $\mathcal{U}$. This completes the proof, due to the assumption that the limit on the left-hand side of~\eqref{eq:final_optimality} exists.
\hfill$\square$

\begin{remark}\label{remark:proof_diff}
Note that the differences between our proof and that of~\cite{calvo2021state} are two-fold: i) In our work, the updates for the dual variables are based on unbiased gradient estimates in~\eqref{eq:unbiased_estimate_prev_exp} as opposed to the reinforcement learning scenario in~\cite{calvo2021state}; and ii) The time averaging operation in our problem formulation happens inside the utility and constraint functions, $\ccalU$ and $\bbg$, respectively, as opposed to the formulation in~\cite{calvo2021state}, where the time averaging operation happens outside the reward functions.
\end{remark}

\section{Convex Hull Relaxation}\label{appx:convex_hull}
Theorem~\ref{thm:main}---as well as Theorem 1 of~\cite{calvo2021state}---demonstrate the convergence of the time average of the primal-dual iterates. This can indeed be shown for any non-convex optimization problem, and it comes from the fact that the primal-dual iterates operate on a \emph{convex hull} relaxation. More precisely, consider the following generic optimization problem:
\begin{subequations}\label{eq:generic_opt}
\begin{alignat}{2}
    &\max_{x} &~& f_0(x),           \\
    &~~~\text{s.t.} &&  f(x) \geq 0,%
\end{alignat}
\end{subequations}
with the Lagrangian $\ccalL(x,\mu) = f_0(x) + \mu f(x)$ for $\mu \in \reals_+$ and the Lagrangian maximizer
\begin{align}\label{eq:L_max_orig}
x^{\star}(\mu) = \arg \max_x \ccalL(x,\mu).
\end{align}
Now, consider a generalized version of~\eqref{eq:generic_opt}, in which the optimization over $x$ is replaced by an optimization over a \emph{measure} $m(x)$, i.e.,
\begin{subequations}\label{eq:generic_opt_measure}
\begin{alignat}{2}
    &\max_{m(x)} &~& \int f_0(x) \, dm(x),           \\
    &~~~\text{s.t.} &&  \int f(x) \, dm(x) \geq 0.%
\end{alignat}
\end{subequations}
The Lagrangian corresponding to~\eqref{eq:generic_opt_measure} is then given by $\ccalL(m(x),\mu) = \int \left(f_0(x) + \mu f(x) \right)  \, dm(x)$ for $\mu \in \reals_+$, with the Lagrangian maximizer
\begin{align}\label{eq:L_max_measure}
m^{\star}(x|\mu) = \arg \max_{m(x)} \ccalL(m(x),\mu).
\end{align}
The Lagrangian maximizing measure in~\eqref{eq:L_max_measure} may not be unique. However, it can be shown that the Lagrangian maximizer in~\eqref{eq:L_max_orig} also corresponds to a Lagrangian maximizing measure in~\eqref{eq:L_max_measure}, or conversely, the set of Lagrangian maximizing measures in~\eqref{eq:L_max_measure} contains a measure that corresponds to the solution of~\eqref{eq:L_max_orig}.

\section{Proof of Theorem~\ref{thm:near_universality_result}}\label{appx:proof_state_augmented}

The proof of feasibility in~\eqref{eq:thm_feasibility_state_augmented} follows similar steps to those in Appendix~\ref{appx:proof_feasibility} and we, therefore, omit it for brevity. As for the near-optimality result in~\eqref{eq:thm_optimality_state_augmented}, we have %we show that the objectives resulting from the state-augmented procedure and  first consider the following lemma.

\begin{align}
&\lim_{T\to\infty} \Bigg| \E \Bigg[ \mathcal{U}\left( \frac{1}{T} \sum_{t=0}^{T-1} \bbf\left(\bbH_t, \bbp\left(\bbH_t, \bbmu_{\lfloor t/T_0 \rfloor}; \bbphi^{\star}\right) \right) \right) \nonumber \\
&\quad\qquad\qquad - \mathcal{U}\left( \frac{1}{T} \sum_{t=0}^{T-1} \bbf(\bbH_t, \bbp(\bbH_t ;\bbtheta_{\lfloor t / T_0 \rfloor})) \right) \Bigg] \Bigg| \\
&\leq \lim_{T\to\infty} \E \ \Bigg| \ \mathcal{U}\left( \frac{1}{T} \sum_{t=0}^{T-1} \bbf\left(\bbH_t, \bbp\left(\bbH_t, \bbmu_{\lfloor t/T_0 \rfloor}; \bbphi^{\star}\right) \right) \right) \nonumber \\
&\quad\qquad\qquad - \mathcal{U}\left( \frac{1}{T} \sum_{t=0}^{T-1} \bbf(\bbH_t, \bbp(\bbH_t ;\bbtheta_{\lfloor t / T_0 \rfloor})) \right) \Bigg| \label{eq:convexity_abs} \\
&\leq \lim_{T\to\infty} \E \ \Bigg\| \ \frac{1}{T} \sum_{t=0}^{T-1} \bigg( \bbf\left(\bbH_t, \bbp\left(\bbH_t, \bbmu_{\lfloor t/T_0 \rfloor}; \bbphi^{\star}\right) \right)  \nonumber \\
&\qquad\qquad\qquad\qquad\quad -  \bbf\left(\bbH_t, \bbp(\bbH_t ;\bbtheta_{\lfloor t / T_0 \rfloor})\right) \bigg) \Bigg\|_{\infty} \label{eq:Lipschitz_U} \\
&\leq  \lim_{T\to\infty} \frac{1}{T} \sum_{t=0}^{T-1} \E \ \Bigg\| \   \bbf\left(\bbH_t, \bbp\left(\bbH_t, \bbmu_{\lfloor t/T_0 \rfloor}; \bbphi^{\star}\right) \right)  \nonumber \\
&\qquad\qquad\qquad\qquad\quad ~ -  \bbf\left(\bbH_t, \bbp(\bbH_t ;\bbtheta_{\lfloor t / T_0 \rfloor})\right) \Bigg\|_{\infty} \label{eq:convexity_infnorm} \\
&\leq M \lim_{T\to\infty} \frac{1}{T} \sum_{t=0}^{T-1} \E \ \Bigg\| \   \bbp\left(\bbH_t, \bbmu_{\lfloor t/T_0 \rfloor}; \bbphi^{\star}\right)  \nonumber \\
&\qquad\qquad\qquad\qquad\qquad\qquad\quad -  \bbp(\bbH_t ;\bbtheta_{\lfloor t / T_0 \rfloor})  \Bigg\|_{\infty} \label{eq:Lipschitz_f} \\
&\leq M \E \left\| \bbp\left(\bbH, \bbmu; \bbphi^{\star}\right) - \bbp(\bbH ;\bbtheta(\bbmu))  \right\|_{\infty} \label{eq:exp_mu} \\
&\leq M \epsilon,\label{eq:universality_proofbound}
\end{align}
where~\eqref{eq:convexity_abs} results from the convexity of the absolute value function,~\eqref{eq:Lipschitz_U} comes from the Lipschitz continuity of the utility $\mathcal{U}$,~\eqref{eq:convexity_infnorm} results from the convexity of the $\ell_{\infty}$ norm,~\eqref{eq:Lipschitz_f} holds because of the the Lipschitz continuity of the performance function $\bbf$,~\eqref{eq:exp_mu} holds since $\{\bbmu_k\}_{k=1}^K$ are assumed to come from the distribution $p_{\bbmu}$, and~\eqref{eq:universality_proofbound} is true due to the near-universality of the parameterization. This, combined with~\eqref{eq:thm_optimality} from Theorem~\ref{thm:main}, completes the proof.
\hfill$\square$%%

\bibliographystyle{IEEEtran}
\bibliography{references}

\vfill

\end{document}